\documentclass[pmlr]{jmlr}

\usepackage{amsfonts}
\usepackage{mathtools}
\usepackage{xcolor}
\usepackage{bbold}
\usepackage{booktabs}
\usepackage{longtable}
\usepackage{cleveref}
\usepackage{hyperref}

\newcommand{\E}[1]{\mathbb{E}\left[#1\right]}
\newcommand{\Ex}[2]{\mathbb{E}_{#1}\left[#2\right]}

\renewcommand{\Pr}[1]{\mathbb{P}\left(#1\right)}

\newcommand{\Tmon}{T_{\mathrm{mon}}}

\title{Conformal Prediction via Transported Beta Laws}

\author{
  \Name{Thiago R. Ramos}
  \Email{thiagorr@ufscar.br}\\
  \addr Federal University of S\~ao Carlos
  \AND
  \Name{Helton Graziadei}
  \Email{helton@ufscar.br}\\
  \addr Federal University of S\~ao Carlos
  \AND
  \Name{Luben M. C. Cabezas}
  \Email{lucruz45.cab@gmail.com}\\
  \addr Federal University of S\~ao Carlos, University of S\~ao Paulo, \\ Inria and Université Grenoble Alpes}

\makeatletter
\renewcommand*{\@jmlrenddoc}{%
  \FloatBarrier
  \phantomsection
  \protected@edef\@currentlabelname{end of \@shorttitle}%
  \label{jmlrend}%
  \global\let\@reprint\@empty
}
\makeatother

\begin{document}

\maketitle

\thispagestyle{empty}
\pagestyle{empty}

\begin{abstract}
Split conformal prediction provides finite-sample marginal coverage under exchangeability, but this guarantee averages over the random calibration sample. We study instead the law of the calibration-conditional coverage induced by a realized conformal threshold. In the continuous i.i.d.\ setting this law is exactly \(\operatorname{Beta}(k,n+1-k)\), so the usual marginal guarantee corresponds to its mean. We take this beta law as a finite-sample reference object and quantify
departures from it using Wasserstein distances on \([0,1]\). The framework yields direct bounds on marginal coverage gaps and on bad-calibration probabilities, and separates different sources of non-i.i.d.\ behavior according to how they deform the beta reference: test-side shift acts through a transport map on the coverage scale, while calibration dependence changes the order-statistic law itself. We instantiate the framework in scale-shift, clustered, and stationary mixing settings, where
the induced deformations can be characterized explicitly or through Berry--Esseen approximations. Simulations on dependent processes confirm that the first-order approximation tracks the empirical Wasserstein distance even at moderate sample sizes.
\end{abstract}


\section{Introduction}

Conformal prediction (CP) provides a general framework for constructing
prediction sets with finite-sample coverage guarantees
\citep{random_world_conformal,papadopoulos2002inductive,lei2014distribution}. In split conformal prediction, a training sample is used to fit a prediction rule and a separate calibration sample is used to choose a score threshold. For calibration scores \(S_1,\ldots,S_n\), with large scores indicating worse conformity, the usual threshold at nominal level \(\gamma\in(0,1)\) is
\[
    \hat q_{n,\gamma}
    =
    S_{(k_\gamma)},
    \qquad
    k_\gamma
    =
    \lceil (n+1)\gamma\rceil,
\]
where \(S_{(1)}\leq\cdots\leq S_{(n)}\) are the order statistics and
\(S_{(k)}=+\infty\) for \(k>n\). Under exchangeability, this construction
satisfies the familiar marginal guarantee
\[
    \gamma
    \leq
    \Pr{S_{n+1}\leq \hat q_{n,\gamma}}
    <
    \gamma+\frac{1}{n+1}.
\]

This guarantee is intentionally marginal: the probability averages over the random calibration sample as well as the test score. Once a calibration sample has been drawn, the threshold is fixed; however, across realizations of the calibration sample, the resulting test coverage is itself a random quantity. By the tower property,
\[
    \Pr{S_{n+1}\leq \hat q_{n,\gamma}}
    =
    \E{
        \Pr{
            S_{n+1}\leq \hat q_{n,\gamma}
            \,\middle|\,
            S_1,\ldots,S_n
        }
    }.
\]
We call the inner probability the \emph{calibration-conditional coverage}. The
usual conformal guarantee controls its mean, but not its sampling distribution.

In the continuous i.i.d.\ case this distribution has an exact and universal
form. For \(1\leq k\leq n\), define
\[
    C_{n,k}
    :=
    \Pr{
        S_{n+1}\leq S_{(k)}
        \,\middle|\,
        S_1,\ldots,S_n
    }.
\]
Then
\[
    C_{n,k}
    \sim
    \operatorname{Beta}(k,n+1-k).
\]
The beta law is therefore not only a route to the standard marginal coverage
bound; it is the finite-sample law of the realized coverage itself
\citep{vovk_conditional_2012,angelopoulos2021gentle}. In a single deployment,
the practitioner draws one calibration sample, computes one threshold, and then
uses that threshold for future predictions. The realized coverage is one draw
from this beta law, while marginal validity constrains only its expectation. For
example, with \(n=30\) and \(\gamma=0.9\), so that
\(k_\gamma=28\), the reference law is
\(\operatorname{Beta}(28,3)\); its lower tail assigns probability about \(0.15\)
to realized coverage below \(0.85\), and about \(0.04\) to realized coverage
below \(0.80\). These bad-calibration probabilities are invisible to a statement
that only controls the mean.

Figure~\ref{fig:beta-bad-calibration} visualizes this lower-tail effect and
its decay with the calibration size.

\begin{figure}[h]
    \centering
    \includegraphics[width=.85\textwidth]{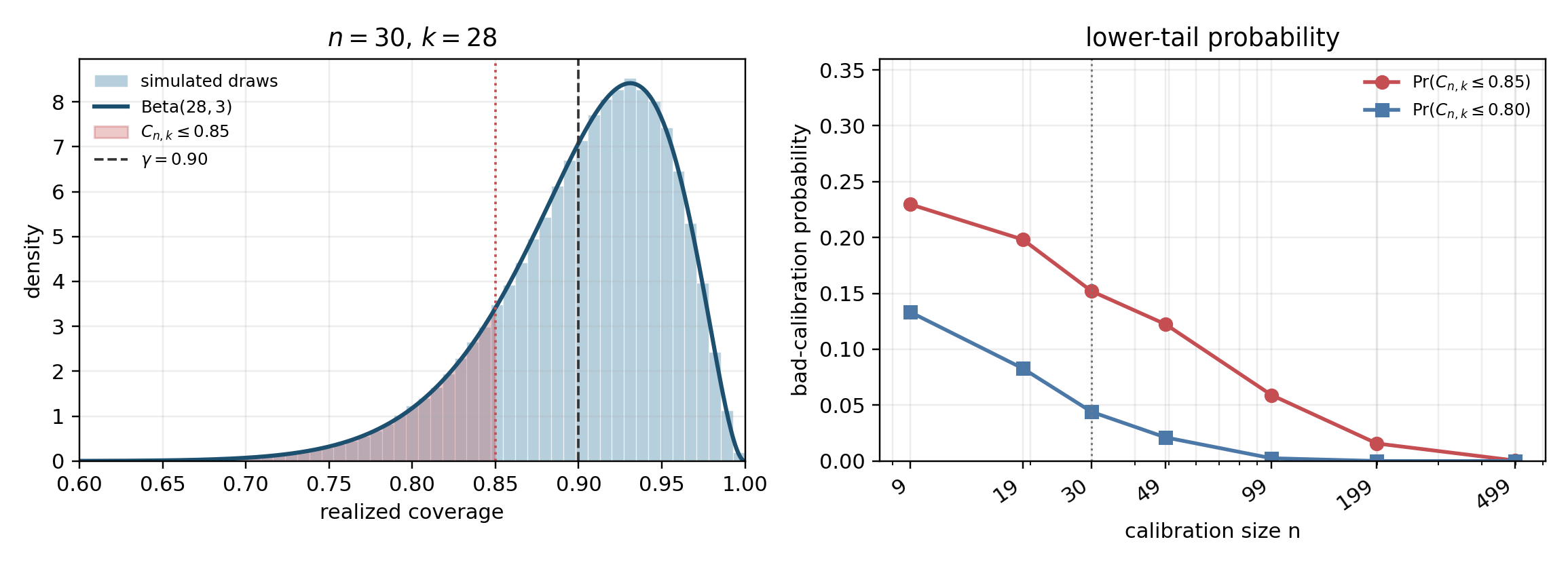}
    \caption{
    Bad-calibration events under  i.i.d.\ beta reference. Left: simulated
    draws of \(C_{n,k}\sim\operatorname{Beta}(28,3)\) for \(n=30\),
    \(\gamma=0.9\), and \(k=28\), with the lower tail
    \(C_{n,k}\leq 0.85\) shaded. Right: simulated lower-tail probabilities
    for \(k=\lceil(n+1)\gamma\rceil\) as the calibration size increases.
    }
    \label{fig:beta-bad-calibration}
\end{figure}

This observation suggests a different way to study conformal prediction beyond
the i.i.d.\ benchmark. Under distribution shift, dependence, or non-identical
score distributions, the calibration-conditional coverage need not follow the
beta law. Nevertheless, it plays the same structural role: marginal coverage is
obtained by averaging this realized coverage. The central question is therefore
how the law of the realized coverage departs from the beta reference, and how
that departure translates into coverage degradation.

We address this question using optimal transport on the coverage scale. Rather
than transporting calibration and test score distributions directly, we compare
the law of the calibration-conditional coverage variable with the beta benchmark
on \([0,1]\). Since the identity map on \([0,1]\) is \(1\)-Lipschitz, a
\(W_1\) comparison on this scale immediately controls the marginal coverage
gap. In this sense, the beta law remains the reference object, and
non-i.i.d.\ mechanisms are understood through the way they transport or deform
it.

Our contributions are as follows. First, we formulate conformal validity through
the law \(\nu_{n,k}\) of the realized calibration-conditional coverage and
introduce Wasserstein beta neighborhoods around the i.i.d.\ benchmark
\(\beta_{n,k}:=\operatorname{Beta}(k,n+1-k)\). Second, we prove that
\(W_1(\nu_{n,k},\beta_{n,k})\) directly bounds the marginal coverage gap and
derive corresponding bounds for bad-calibration probabilities. Third, we show
how different departures from the i.i.d.\ benchmark act on the beta reference:
test-side shift transports the law on the coverage scale, while calibration
dependence changes the order-statistic law itself. Counting-process and
Berry--Esseen arguments then provide concrete ways to quantify these
deformations. The examples are used to illustrate this framework rather than to
define separate procedures.

\subsection{Related work}

\paragraph{The beta law in conformal prediction.}
Under continuous i.i.d.\ scores, the calibration-conditional coverage of split
conformal prediction follows a beta distribution
\citep{vovk_conditional_2012,angelopoulos2021gentle}. The i.i.d.\ assumption
can be relaxed to exchangeability. \citet{marques2025universal} show that, for
a test sample exchangeable with the calibration scores, the coverage indicators
\(Z_i=\mathbf 1\{S_{n+i}\leq S_{(k)}\}\) form an exchangeable sequence; by de
Finetti's theorem, the empirical coverage converges almost surely as the test
sample size grows to a random limit, which a combinatorial argument identifies
as a beta law. We work in the i.i.d.\ framing for notational simplicity and
because it gives the probability integral transform used in our Wasserstein
analysis.

\paragraph{Conformal prediction and optimal transport.}
A growing line of work applies optimal transport tools to conformal prediction
under distribution shift, but the transported object differs from ours.
\citet{xu2025wasserstein} bound the coverage gap under joint covariate and
concept shift using the \(1\)-Wasserstein distance between calibration and test
score distributions, and use the bound as a training regularizer.
\citet{correia2025nonexchangeable} derive coverage-gap bounds formulated via
optimal transport distances on the score space and use unlabeled test data to
attenuate the gap. \citet{aolaritei2025conformal} model distribution shift through
Levy--Prokhorov ambiguity sets propagated through the score function, reducing a
high-dimensional shift to a one-dimensional problem on the score scale. These
approaches compare score laws, transported weights, or ambiguity sets. Our
comparison takes place after calibration: the transported object is the law of
the realized coverage itself on \([0,1]\). This is what makes the
Kantorovich--Rubinstein duality immediately yield a marginal coverage-gap
bound, without an intermediate quantile-level argument.

\paragraph{Conformal prediction beyond exchangeability.}
Several other lines of work quantify how violations of exchangeability affect
split conformal coverage. \citet{barber2023conformal} develop
non-exchangeable conformal prediction via weighted quantiles and nonsymmetric
algorithms, with finite-sample bounds that degrade in the total variation
distance to a reference exchangeable law. \citet{oliveira2024split} show that
vanilla split conformal prediction remains valid for a broad class of
non-exchangeable processes, including time series and spatiotemporal data, at
the cost of a small coverage penalty controlled by concentration and decoupling
arguments. In contrast, we replace total-variation or mixing-type penalties on
the data-generating law with \(W_1\) distances between realized-coverage laws.

A related asymptotic literature studies the joint behavior of conformal
coverage across a test sample. \citet{gazin2024transductive} show that the
joint distribution of conformal \(p\)-values in a full conformal setting is a
Pólya urn and prove a concentration inequality for their empirical c.d.f.\
under exchangeability. \citet{gazin2025asymptotics} identify the asymptotic
distribution of the false coverage proportion as \(n,m\to\infty\), recovering
the Kolmogorov law after rescaling, with extensions to weighted conformal and
covariate shift.  By contrast, our object is finite-sample and fixed-quantile: the law of ($C_{n,k}$), together with its non-i.i.d.\ analogue. The Pólya urn and Brownian bridge representations offer complementary large-test-sample descriptions; here we instead study how the finite-sample beta law itself is transported under shift and dependence.

\section{Background}\label{sec:background}
This section develops the theoretical background required for the remainder of the paper.

\subsection{Conformal prediction}

We recall the basic split conformal construction and the two classical facts that will serve as reference points throughout the paper. The first is the usual finite-sample marginal coverage guarantee. The second is a more refined description of the random coverage level induced by the realized calibration sample. 

Let $S \sim F_S$ be a continuous score distribution, and let
$S_1,\ldots,S_n$ denote calibration scores on this scale. In the i.i.d.\ case,
these scores form an independent sample from $F_S$.
Define the \emph{adjusted empirical conformal $\gamma$-quantile} as
\[
    \hat{q}_{n,\gamma}
    =
    S_{\!\left(\lceil \gamma(n+1)\rceil\right)},
\]
with the convention that $S_{(k)} = +\infty$ whenever $k > n$.
\footnote{This definition---using $n+1$ in place of $n$---is standard in
conformal prediction, as it guarantees finite-sample marginal coverage of at least $\gamma$: $\Pr{S_{n+1} \leq \hat{q}_{n,\gamma}} \geq \gamma$.}

The conformal threshold is therefore an empirical order statistic of the
calibration scores. The use of $\lceil \gamma(n+1)\rceil$ rather than the
usual empirical quantile index is the finite-sample correction that makes the
split conformal guarantee hold without asymptotics.

\begin{theorem}[Split Conformal Marginal Coverage]\label{thm:standard_CP}
Assume that $S_1,\ldots,S_n,S_{n+1}$ are exchangeable and that there are no
ties almost surely. Let
$
    k_\gamma
    =
    \lceil (n+1)\gamma\rceil
$
and
$
    \hat q_{n,\gamma}
    =
    S_{(k_\gamma)},
$
with the convention that $S_{(k)}=+\infty$ whenever $k>n$. Then
\[
    \gamma
    \leq
    \Pr{S_{n+1}\leq \hat q_{n,\gamma}}
    <
    \gamma+\frac{1}{n+1}.
\]
\end{theorem}

The proof of this statement is elementary and relies on the uniformity of
ranks. This is the classical finite-sample marginal coverage guarantee of
split conformal prediction. It is the benchmark statement: under
exchangeability, the conformal threshold covers a fresh score with probability
at least $\gamma$, up to the unavoidable discretization error of order
$(n+1)^{-1}$.

For our purposes, however, the marginal guarantee is only the first layer of
the story. Once the calibration sample is realized, the order statistic
$S_{(k)}$ is fixed, and the coverage of this realized threshold is itself a
random quantity as the calibration sample varies. We next describe the
distribution of this calibration-conditional coverage for a fixed order statistic.

\begin{proposition}\label{prop:conditional_coverage_beta}
Let $S \sim F_S$ be a continuous random variable and let
$\{S_i\}_{i=1}^{n}$ be an i.i.d.\ sample from $F_S$. For
$1\leq k\leq n$, define
\[
    C_{n,k}
    =
    \Pr{
        S_{n+1}\leq S_{(k)}
        \,\middle|\,
        S_1,\ldots,S_n
    }.
\]
Then
\[
    C_{n,k}
    =
    F_S(S_{(k)}) \sim
    \operatorname{Beta}(k,n+1-k).
\]

\end{proposition}

The proof is given in Appendix~\ref{app:proof_conditional_coverage_beta}.
This proposition gives a conditional refinement of the marginal CP guarantee.
It says that, in the continuous i.i.d.\ setting, the realized coverage of the threshold $S_{(k)}$ is not arbitrary: its law is exactly the beta law
associated with the corresponding order statistic.

Throughout the paper, we write
\[
    \beta_{n,k}
    :=
    \operatorname{Beta}(k,n+1-k)
\]
for this i.i.d.\ reference law.

In particular, letting
$
    k_\gamma = \lceil (n+1)\gamma\rceil,
$
Proposition~\ref{prop:conditional_coverage_beta} gives us
\[
\begin{aligned}
    \Pr{S_{n+1}\leq \hat{q}_{n,\gamma}}
    &=
    \E{
        \Pr{
            S_{n+1}\leq \hat{q}_{n,\gamma}
            \,\middle|\,
            S_1,\ldots,S_n
        }
    } \\
    &=
    \E{C_{n,k_\gamma}} \\
    &=
        \frac{\lceil (n+1)\gamma\rceil}{n+1}
    \in
    \left[
        \gamma,
        \gamma+\frac{1}{n+1}
    \right],
\end{aligned}
\]
which recovers the classic CP upper and lower bounds.

The same beta law also gives a precise way to discuss bad calibration. A bad
calibration event occurs when the realized coverage produced by the calibration
sample is substantially below the nominal level. The marginal split conformal
guarantee controls only the mean of \(C_{n,k}\); the lower tail of the beta law
controls how often an unfavorable calibration sample occurs.

\begin{theorem}[i.i.d.\ Bad Calibration Probabilities]
\label{thm:iid_bad_calibration}
Under the assumptions of Proposition~\ref{prop:conditional_coverage_beta}, let
\(B_{n,k}\sim\beta_{n,k}\). Then, for every \(t\in[0,1]\),
\[
    \Pr{C_{n,k}\leq t}
    =
    \Pr{B_{n,k}\leq t}.
\]
In particular, if
$
    k_\gamma=\lceil(n+1)\gamma\rceil\leq n,
$
then, for every \(\eta\in(0,\gamma)\),
\[
    \Pr{C_{n,k_\gamma}\leq \gamma-\eta}
    =
    \Pr{B_{n,k_\gamma}\leq \gamma-\eta}.
\]
\end{theorem}

This follows immediately from Proposition~\ref{prop:conditional_coverage_beta}.
It identifies the beta lower tail as the i.i.d.\ reference probability of a bad calibration event. Later, when the realized coverage law is no longer exactly beta, our Wasserstein bounds will compare its lower tail to this same reference
quantity.

Thus, the classical split conformal guarantee is recovered as the mean of this beta law. The rest of the paper takes this beta distribution as the reference law: departures from the i.i.d.\ setting will be measured by how much the realized
conditional coverage law deviates from this beta benchmark.

\subsection{Optimal Transport and Wasserstein Distance}
\label{subsec:wasserstein_background}

We recall the few facts about Wasserstein distances that will be used below
\citep{ambrosio2005gradient,villani2009optimal,peyre2019computational,santambrogio2015optimal}.
Let \((\mathcal X,d)\) be a metric space. For \(p\geq 1\), the
\(p\)-Wasserstein distance between probability laws \(\mu\) and \(\nu\) is
\[
    W_p(\mu,\nu)
    =
    \left(
    \inf_{\pi\in\Pi(\mu,\nu)}
    \int_{\mathcal X\times\mathcal X}
        d(x,y)^p
    \,d\pi(x,y)
    \right)^{1/p},
\]
where \(\Pi(\mu,\nu)\) is the set of couplings with marginals \(\mu\) and
\(\nu\). On the space of probability laws with finite \(p\)-th moment, \(W_p\)
is a metric. In particular, when \(\mathcal X=[0,1]\), the moment condition is
automatic.

In this paper, the laws compared by Wasserstein distances are laws of calibration-conditional coverage variables, supported on \([0,1]\). In this setting, with \(d(u,v)=|u-v|\), the definition becomes
\[
    W_p(\mu,\nu)
    =
    \left(
    \inf_{\pi\in\Pi(\mu,\nu)}
    \int_{[0,1]^2}
        |u-v|^p
    \,d\pi(u,v)
    \right)^{1/p}.
\]
We first record two elementary consequences of the definition.

\begin{proposition}[Coupling Upper Bound]
\label{prop:w1_coupling_upper_bound}
Fix \(p\geq 1\). Let \(\mu\) and \(\nu\) be probability laws on \([0,1]\). If
\(X\sim\mu\) and \(Y\sim\nu\) are defined on the same probability space, then
\[
    W_p(\mu,\nu)
    \leq
    \left(\E{|X-Y|^p}\right)^{1/p}.
\]
\end{proposition}

This bound is our basic way of turning a concrete coupling into a Wasserstein
estimate. The next fact explains how estimates at different Wasserstein orders
are related.

\begin{proposition}[Monotonicity in \(p\)]
\label{prop:wasserstein_monotonicity}
Let \(\mu\) and \(\nu\) be probability laws on \([0,1]\). If
\(1\leq p\leq q\), then
\[
    W_p(\mu,\nu)
    \leq
    W_q(\mu,\nu).
\]
\end{proposition}

This monotonicity lets us state some intermediate estimates in the
Wasserstein order that is most convenient and then read them on the \(W_1\)
scale used for coverage.

For estimates involving expectations, we will also use the following dual form
of \(W_1\).

\begin{proposition}[Dual Representation and Mean Bound]
\label{prop:w1_dual_mean_bound}
If $\mu$ and $\nu$ are probability measures on $[0,1]$, then
\[
    W_1(\mu,\nu)
    =
    \sup_{\|f\|_{\mathrm{Lip}}\leq 1}
    \left|
        \int f(u)\,d\mu(u)
        -
        \int f(u)\,d\nu(u)
    \right|.
\]
In particular,
\[
    \left|
        \int u\,d\mu(u)
        -
        \int u\,d\nu(u)
    \right|
    \leq
    W_1(\mu,\nu).
\]
\end{proposition}

The first identity in Proposition~\ref{prop:w1_dual_mean_bound} is the
Kantorovich--Rubinstein duality
\citep[Eq.~(3.1)]{santambrogio2015optimal}. The second follows by taking the
identity map \(u\mapsto u\), which is \(1\)-Lipschitz on \([0,1]\).
Together with Proposition~\ref{prop:wasserstein_monotonicity}, this means that
any \(W_p\) bound with \(p\geq 1\) can be converted into the corresponding
mean bound through \(W_1\) when needed.

\begin{proposition}[Optimality of Monotone Transport]
\label{prop:monotone_transport}
Fix \(p\geq 1\). Let \(\mu\) and \(\nu\) be probability laws on \(\mathbb R\)
with finite \(p\)-th moments. Assume that \(F_\mu\) is continuous, and let
\(F_\nu^{-1}\) denote the generalized inverse of \(F_\nu\). Define
\[
    \Tmon(x)
    =
    F_\nu^{-1}(F_\mu(x)).
\]
If \(X\sim\mu\), then \(\Tmon(X)\sim\nu\), and the coupling
\((X,\Tmon(X))\) is optimal. Consequently,
\[
    W_p^p(\mu,\nu)
    =
    \E{\left|X-\Tmon(X)\right|^p}.
\]
\end{proposition}

This is the one-dimensional monotone rearrangement theorem
\citep[Thm.~2.9]{santambrogio2015optimal}. It gives an explicit optimal map,
which will be useful when a distorted realized-coverage law can be transported
back to its beta reference.

\begin{proposition}[One-Dimensional Formula]
\label{prop:w1_one_dimensional_formula}
Let \(\mu\) and \(\nu\) be probability laws on \([0,1]\), with distribution
functions \(F_\mu\) and \(F_\nu\). Then
\[
    W_1(\mu,\nu)
    =
    \int_0^1
    \left|
        F_\mu(t)-F_\nu(t)
    \right|
    dt.
\]
\end{proposition}

Proposition~\ref{prop:w1_one_dimensional_formula} will be useful when the
calibration-conditional coverage law is described through an order statistic or through the
associated counting process.

Although our main objects are supported on \([0,1]\), some asymptotic
calculations below compare Gaussian limits on \(\mathbb R\). We will use the
following closed form for centered normal laws, where \(W_1\) is computed with
the Euclidean distance.

\begin{lemma}[Centered Normal Laws]
\label{lem:w1_centered_normals}
For any \(\sigma_1,\sigma_2\geq 0\),
\[
    W_1\!\left(
        N(0,\sigma_1^2),
        N(0,\sigma_2^2)
    \right)
    =
    \sqrt{\frac{2}{\pi}}
    |\sigma_1-\sigma_2|.
\]
\end{lemma}

\section{A Transportation Framework for Conformal Prediction}\label{sec:method}

In the previous section, we identified \(\beta_{n,k}\) as the law of the
realized coverage induced by the \(k\)-th conformal order statistic in the
continuous i.i.d.\ setting. We now use this beta law as a reference object. A
calibration and test mechanism induces a law for its realized coverage, and the
main question is how much this law differs from the i.i.d.\ beta benchmark.

Fix \(1\leq k\leq n\), and let
\[
    \mathcal C_n=(T_1,\ldots,T_n)
\]
be the calibration scores under a joint law \(P\). Let \(\widetilde T\) denote
the test score.\footnote{Throughout the text, \(S_i\) denotes i.i.d.\ scores,
while \(T_i\) denotes scores that need not be i.i.d.} This notation does not tie
the test score to the next index after calibration; for example, in a
time-series setting one may take \(\widetilde T=T_{n+h}\), with \(h\) fixed or
depending on \(n\). Let the calibration order statistics be
\[
    T_{(1)}\leq\cdots\leq T_{(n)}.
\]
The marginal coverage of the threshold \(T_{(k)}\) is
\[
    \operatorname{Cov}(k)
    :=
    \Pr{\widetilde T\leq T_{(k)}}.
\]
We write it as the expectation of the coverage realized by the particular
calibration sample:
\[
    D_{n,k}
    :=
    \Pr{
        \widetilde T\leq T_{(k)}
        \,\middle|\,
        \mathcal C_n
    }.
\]
Thus
\[
    \operatorname{Cov}(k)=\E{D_{n,k}}.
\]
The law of this random realized coverage is
\[
    \nu_{n,k}:=\mathcal L(D_{n,k}).
\]
In the continuous i.i.d.\ reference case,
Proposition~\ref{prop:conditional_coverage_beta} gives
\(\nu_{n,k}=\beta_{n,k}\). The framework below keeps \(\beta_{n,k}\) fixed as
the benchmark and measures the deformation from \(\beta_{n,k}\) to
\(\nu_{n,k}\) on the coverage scale \([0,1]\). This leads to the following
convenient notation.

\begin{definition}[Wasserstein Beta Neighborhood]
\label{def:wasserstein_beta_neighborhood}
Let \(\mathcal P([0,1])\) denote the set of probability laws on \([0,1]\). For
\(p\geq 1\) and \(\rho\geq 0\), define
\[
    \mathcal B_p(\beta_{n,k},\rho)
    :=
    \left\{
        \nu\in\mathcal P([0,1]):
        W_p(\nu,\beta_{n,k})\leq \rho
    \right\}.
\]
We say that the realized coverage law is in the \(p\)-Wasserstein beta
neighborhood of radius \(\rho\) when
\[
    \nu_{n,k}\in\mathcal B_p(\beta_{n,k},\rho).
\]
\end{definition}

The radius \(\rho\) quantifies how far the realized coverage law is from the
i.i.d.\ beta law. When \(\rho=0\), the realized coverage law is exactly the beta
reference; for positive \(\rho\), it is a transported or perturbed version of
that reference. To make this geometric picture concrete, consider the
contaminated law $\nu_{\pi,c} = (1-\pi)\beta_{n,k} + \pi\delta_c$, which
replaces a fraction $\pi$ of the beta mass by a point mass at coverage value
$c \in [0,1]$. A direct computation via Proposition ~\ref{prop:w1_one_dimensional_formula} gives
\[
    W_1(\nu_{\pi,c}, \beta_{n,k}) 
    = \pi \int_0^1 |\mathbf{1}\{x \geq c\} - F_{\beta}(x)|\,dx,
\]
so the radius grows linearly in $\pi$ at a rate determined by how far $c$
sits from the bulk of $\beta_{n,k}$. Figure~\ref{fig:fig_w1_ball} displays this as a function of $(\pi, c)$, with each contour tracing the boundary of the Wasserstein ball $\mathcal{B}_1(\beta_{n,k}, \rho)$ for a different radius $\rho$. The shape of these contours reflects the geometry of the contamination: for $c$ close to $\gamma$, the point mass lands near the bulk of $\beta_{n,k}$ and transport is cheap, so the contour allows a larger contamination fraction $\pi$ before the radius $\rho$ is exceeded. As $c$ moves away from $\gamma$, the transport cost per unit of $\pi$ increases, and the contour bends inward, permitting only a smaller $\pi$ for the same budget $\rho$. Each contour therefore defines a feasible region in $(\pi, c)$: staying inside it is precisely the condition $\nu_{\pi,c} \in \mathcal{B}_1(\beta_{n,k}, \rho)$.

\begin{figure}[h]
    \centering
    \includegraphics[width=0.7\linewidth]{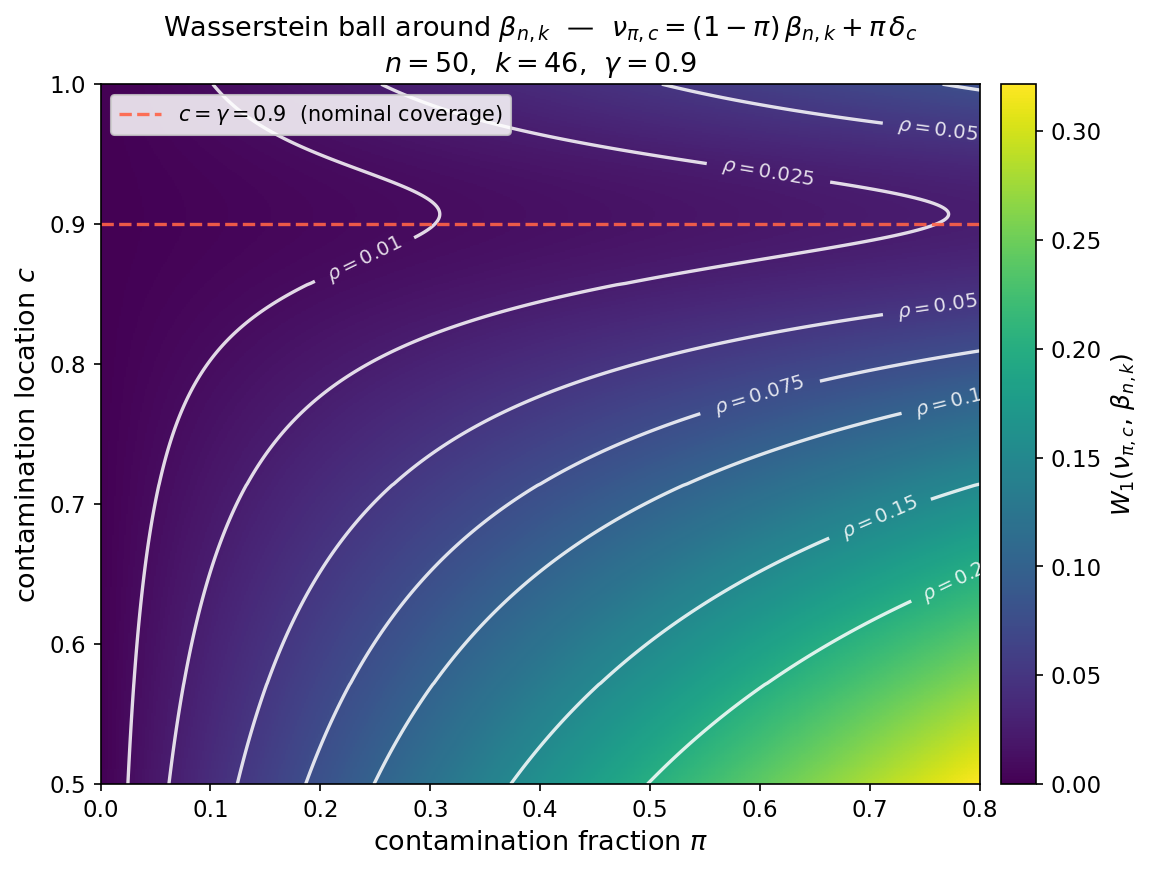}
    \caption{Wasserstein radius $W_1(\nu_{\pi,c}, \beta_{n,k})$ for the 
    contaminated law $\nu_{\pi,c} = (1-\pi)\beta_{n,k} + \pi\delta_c$,
    with $n = 50$, $k = 46$, and $\gamma = 0.9$.
    Each white contour traces the boundary of $\mathcal{B}_1(\beta_{n,k},\rho)$
    for a given $\rho$, defining the set of contaminations $(\pi, c)$ 
    compatible with that transport budget. Contamination near $c = \gamma$ 
    allows a larger fraction $\pi$ for the same radius, while contamination 
    away from $\gamma$ forces $\pi$ to be small — the contour shape 
    directly encodes this trade-off.}
    \label{fig:fig_w1_ball}
\end{figure}

\subsection{Coverage Guarantees}

The first consequence of a Wasserstein comparison is a direct bound on marginal
coverage. This is the basic reason for working on the coverage scale.

\begin{theorem}[Wasserstein Coverage Gap]
\label{thm:w1_coverage_gap}
Fix \(1\leq k\leq n\), and define \(D_{n,k}\), \(\nu_{n,k}\), and
\(\operatorname{Cov}(k)\) as above. Then
\[
    \left|
        \operatorname{Cov}(k)
        -
        \frac{k}{n+1}
    \right|
    \leq
    W_1(\nu_{n,k},\beta_{n,k}).
\]
Consequently, if \(\nu_{n,k}\in\mathcal B_p(\beta_{n,k},\rho)\) for some
\(p\geq 1\), then
\[
    \left|
        \operatorname{Cov}(k)
        -
        \frac{k}{n+1}
    \right|
    \leq
    \rho.
\]
If \(k_\gamma=\lceil(n+1)\gamma\rceil\leq n\), then
\begin{align*}
    \left|
        \operatorname{Cov}(k_\gamma)
        -
        \gamma
    \right|
    &\leq
    \rho+\frac{1}{n+1}.
\end{align*}

\end{theorem}

The proof is given in Appendix~\ref{app:proof_w1_coverage_gap}. The theorem
says that once the realized coverage law is close to the beta benchmark in
Wasserstein distance, the usual conformal coverage level degrades by at most
the same amount, up to the standard discretization term.

The previous result controls the average of \(D_{n,k}\). We can also ask for
the probability that a realized calibration sample produces unusually low
coverage. A Wasserstein beta neighborhood gives a comparison with the lower tail
of the beta reference.

\begin{theorem}[Bad Calibration Probabilities]
\label{thm:bad_calibration}
Let \(B_{n,k}\sim\beta_{n,k}\). If
\[
    \nu_{n,k}\in\mathcal B_p(\beta_{n,k},\rho)
\]
for some \(p\geq 1\), then, for every \(t\in[0,1]\) and every
\(\varepsilon>0\),
\[
    \Pr{D_{n,k}\leq t}
    \leq
    \Pr{B_{n,k}\leq t+\varepsilon}
    +
    \left(\frac{\rho}{\varepsilon}\right)^p.
\]
\end{theorem}

The proof is given in Appendix~\ref{app:proof_bad_calibration}. The beta term is
the probability of a bad calibration event in the i.i.d.\ reference model; the
additional term accounts for the transport needed to move the beta law to the
realized coverage law. The Markov penalty is useful when only an average
transport radius is available, but it can be conservative, especially on the
\(W_1\) scale.

When the deformation from the beta law can be bounded uniformly, the bad
calibration comparison takes a sharper form. Let
\[
    \Tmon
    =
    F_{\nu_{n,k}}^{-1}\circ F_{\beta_{n,k}}
\]
be the monotone transport map from \(\beta_{n,k}\) to \(\nu_{n,k}\), and let
\(B_{n,k}\sim\beta_{n,k}\). If
\[
    |\Tmon(B_{n,k})-B_{n,k}|
    \leq
    \rho
    \qquad \text{a.s.},
\]
then for every \(t\in[0,1]\),
\[
    \Pr{D_{n,k}\leq t}
    \leq
    \Pr{B_{n,k}\leq t+\rho}.
\]
Indeed, under this coupling \(D_{n,k}\stackrel d=\Tmon(B_{n,k})\), and
\(\Tmon(B_{n,k})\leq t\) implies \(B_{n,k}\leq t+\rho\).

Thus a uniform transport bound shifts the beta lower tail without the Markov
penalty appearing in Theorem~\ref{thm:bad_calibration}. This is the cleanest
version of the transported-beta picture, and we use it whenever an explicit
transport map is available.

\subsection{Decoupled Test Scores and Transported Beta Laws}

A particularly useful simplification occurs when the test score is decoupled
from the calibration sample. This assumption is natural in independent holdout
evaluation and in distribution-shift settings with separate calibration and test
batches. It is also a useful reduction for temporally dependent data when the
test point is sufficiently separated from the calibration block and the
dependence decays with the lag.

\begin{proposition}[Decoupled Test Scores]
\label{prop:decoupled_test_scores}
Assume that \(\widetilde T\) is independent of \(\mathcal C_n\), and let
\(F_{\widetilde T}\) be its distribution function. Then
\[
    D_{n,k}
    =
    F_{\widetilde T}(T_{(k)}),
    \qquad
    \nu_{n,k}
    =
    \mathcal L\!\left(F_{\widetilde T}(T_{(k)})\right).
\]
Suppose that the calibration scores have a common continuous marginal
distribution function \(F_T\), and let \(F_T^{-1}\) denote its generalized
inverse. Define
\[
    U_i:=F_T(T_i),
    \qquad
    i=1,\ldots,n.
\]
Then the variables \(U_i\) are marginally uniform on \([0,1]\), possibly
dependent, and, for
$
    h:=F_{\widetilde T}\circ F_T^{-1},
$
we have
\[
    D_{n,k}\stackrel d=h(U_{(k)}),
    \qquad
    \nu_{n,k}=h_{\#}\mathcal L(U_{(k)}).
\]
\end{proposition}

The proof is given in Appendix~\ref{app:proof_decoupled_test_scores}. In the
no-shift case \(F_{\widetilde T}=F_T\), the map \(h\) is the identity and
\(\nu_{n,k}=\mathcal L(U_{(k)})\), so Theorem~\ref{thm:w1_coverage_gap}
reduces coverage analysis to
\[
    W_1\!\left(
        \mathcal L(U_{(k)}),
        \beta_{n,k}
    \right).
\]
The following proposition records the counting-process representation of this
distance.

\begin{proposition}[Counting Formula]
\label{prop:dependent_calibration_counting}
Let \(U_1,\ldots,U_n\) be random variables on \([0,1]\). Assume each \(U_i\) is uniform; independence is not required. Let
\[
    F_n(t)
    :=
    \frac1n\sum_{i=1}^n \mathbf{1}\{U_i\leq t\},
    \qquad
    N_n(t):=nF_n(t).
\]
Then, for every \(t\in[0,1]\) and \(1\leq k\leq n\),
\[
    \{U_{(k)}\leq t\}
    =
    \left\{
        N_n(t)
        \geq
        k
    \right\}.
\]
Consequently, if \(Z_{n,t}\sim\operatorname{Bin}(n,t)\), then
\[
    W_1\!\left(
        \mathcal L(U_{(k)}),
        \beta_{n,k}
    \right)
    =
    \int_0^1
    \left|
        \Pr{N_n(t)\geq k}
        -
        \Pr{Z_{n,t}\geq k}
    \right|
    \,dt.
\]
\end{proposition}

The proof is given in Appendix~\ref{app:proof_dependent_calibration_counting}.
Thus the Wasserstein distance to the beta law is exactly the integrated
difference between the dependent empirical-count tail and the binomial tail.

\subsection{Berry--Esseen Transport for Dependent Calibration}
\label{subsec:berry_esseen_transport_dependent_calibration}

In many dependent calibration settings, the law of \(U_{(k)}\) is not exactly
beta, but it admits a Berry--Esseen approximation. This is enough to control its
Wasserstein distance to the beta reference.

Let \(U_1,\ldots,U_n\) be marginally uniform random variables on \([0,1]\), not
necessarily independent, and fix \(\gamma\in(0,1)\). 
Here \(\Phi\) denotes the standard normal distribution function. Suppose that,
for some \(\tau_\gamma>0\) and some deterministic sequence \(a_n\to0\),
\[
    \sup_{x\in\mathbb R}
    \left|
        \Pr{
            \sqrt n\{U_{(k_\gamma)}-\gamma\}\leq x
        }
        -
        \Phi\!\left(\frac{x}{\tau_\gamma}\right)
    \right|
    \leq
    a_n.
\]
This assumption says that the calibration order statistic has a Berry--Esseen
law with asymptotic standard deviation \(\tau_\gamma\). The i.i.d.\ beta
reference has the same form with standard deviation \(\sqrt{\gamma(1-\gamma)}\).

\begin{proposition}[Quantile-to-Beta Transport]
\label{prop:quantile_be_beta_transport}
If the preceding Berry--Esseen bound holds, then
\[
    W_1\!\left(
        \mathcal L(U_{(k_\gamma)}),
        \beta_{n,k_\gamma}
    \right)
    \leq
    a_n
    +
    \sqrt{\frac{2}{\pi n}}
    \left|
        \tau_\gamma-\sqrt{\gamma(1-\gamma)}
    \right|
    +
    O(n^{-1/2}).
\]
In particular, if \(a_n=O(n^{-1/2})\), then
\[
    W_1\!\left(
        \mathcal L(U_{(k_\gamma)}),
        \beta_{n,k_\gamma}
    \right)
    \leq
    \sqrt{\frac{2}{\pi n}}
    \left|
        \tau_\gamma-\sqrt{\gamma(1-\gamma)}
    \right|
    +
    O(n^{-1/2}).
\]
\end{proposition}

The proof is given in
Appendix~\ref{app:proof_quantile_be_beta_transport}.

In the next section, we apply this framework in several concrete settings. The
examples differ in how the beta law is deformed, but the logic is the same:
identify the realized coverage law, compare it with \(\beta_{n,k}\), and
translate that comparison into a coverage statement.

\section{Applications}

\subsection{Exact Transported Beta Laws under Distribution Shift}

We first consider a decoupled distribution-shift setting. The calibration scores are i.i.d.\ from a continuous distribution function \(F_T\), while the test score is independent of the calibration sample but follows a possibly different distribution function \(F_{\widetilde T}\). Let \(F_T^{-1}\) denote the
generalized inverse, or quantile function, of the calibration score
distribution. The conformal threshold \(T_{(k)}\) is therefore generated under the calibration law, but its coverage is evaluated under the test law.

For a realized calibration sample, the realized test coverage is
\[
    D_{n,k}
    :=
    \Pr{
        \widetilde T\leq T_{(k)}
        \,\middle|\,
        T_1,\ldots,T_n
    }
    =
    F_{\widetilde T}(T_{(k)}).
\]
Since the calibration sample is i.i.d.\ with distribution function \(F_T\), the probability integral transform gives
\[
    B_{n,k}
    :=
    F_T(T_{(k)})
    \sim
    \beta_{n,k}.
\]
Thus the shifted realized coverage can be written on the beta scale as
\[
    D_{n,k}
    =
    F_{\widetilde T}\!\left(F_T^{-1}(B_{n,k})\right).
\]
Equivalently, if
$
    h:=F_{\widetilde T}\circ F_T^{-1},
$
then
\[
    D_{n,k}\stackrel{d}{=}h(B_{n,k}),
    \qquad
    \nu_{n,k}=h_{\#}\beta_{n,k}.
\]
Distribution shift therefore does not alter the beta order-statistic
mechanism itself. It transports the beta law through the map relating calibration and test score distributions. When \(F_{\widetilde T}=F_T\), the map \(h\) is the identity and \(\nu_{n,k}=\beta_{n,k}\), recovering the i.i.d.\ reference case. In general, the relevant comparison is
\[
    W_1(\nu_{n,k},\beta_{n,k})
    =
    W_1\!\left(
        (F_{\widetilde T}\circ F_T^{-1})_{\#}\beta_{n,k},
        \beta_{n,k}
    \right).
\]
We illustrate the distribution-shift mechanism in a setting where the transport map is explicit.

\subsubsection{Example: half-normal scale shift}

The example is motivated by the usual split conformal regression score
\[
    T = |Y-\widehat{\mu}(X)|.
\]
If the prediction error is centered Gaussian, then \(T\) has a half-normal distribution. Let
\[
    T_1,\ldots,T_n \overset{\mathrm{iid}}{\sim} |N(0,\sigma_T^2)|,
    \qquad
    \widetilde T\sim |N(0,\sigma_{\widetilde T}^2)|,
\]
with \(\widetilde T\) independent of the calibration sample, and define the scale ratio
\[
    r:=\frac{\sigma_{\widetilde T}}{\sigma_T}.
\]
Writing
$
    q(u):=\Phi^{-1}\!\left(\frac{u+1}{2}\right),
$
the calibration-to-test transport map is
\[
    h_r(u)
    =
    2\Phi\!\left(\frac{q(u)}{r}\right)-1.
\]
Thus, if $
    B_{n,k}:=F_{\sigma_T}(T_{(k)})\sim \beta_{n,k},
$
then the realized coverage under the shifted test distribution satisfies
\[
    D_{n,k}
    =
    F_{\sigma_{\widetilde T}}(T_{(k)})
    \stackrel d=
    h_r(B_{n,k}),
    \qquad
    \nu_{n,k}^{(r)}=(h_r)_\#\beta_{n,k}.
\]
Therefore, in the half-normal scale-shift model, distribution shift does not alter the beta order-statistic law on the calibration scale; it transports this law through the deterministic map \(h_r\). Details on the derivation of \(h_r\), its inverse, and the corresponding density transformation are given in Appendix~\ref{app:half_normal_scale_shift_details}.

The direction of the deformation is determined by \(r\). If \(r>1\), the test scores are more dispersed than the calibration scores and \(h_r(u)<u\), leading to undercoverage. If \(r<1\), then \(h_r(u)>u\), leading to overcoverage. Moreover,
\[
    \operatorname{Cov}(k)
    =
    \mathbb E[h_r(B_{n,k})],
\]
so that
\[
    \operatorname{Cov}(k)-\frac{k}{n+1}
    =
    \mathbb E[h_r(B_{n,k})-B_{n,k}].
\]
In this example, the Wasserstein coverage-gap bound is:
\[
    W_1(\nu_{n,k}^{(r)},\beta_{n,k})
    =
    \left|
        \operatorname{Cov}(k)-\frac{k}{n+1}
    \right|.
\]

Figure~\ref{fig:half-normal-beta-deformation} illustrates the induced deformation for \(n=30\), \(\gamma=0.9\), and
\(k=\lceil(n+1)\gamma\rceil=28\), so that the i.i.d.\ reference law is \(\operatorname{Beta}(28,3)\). The curve \(r=1\) corresponds to this reference law. As \(r\) increases above one, the realized-coverage law is transported to the left, increasing the probability of low calibration-conditional coverage. The vertical lines mark the nominal level \(\gamma=0.9\) and the bad-calibration threshold \(\gamma-\eta=0.85\), with \(\eta=0.05\).

\begin{figure}[t]
    \centering
    \includegraphics[width=.7\textwidth]{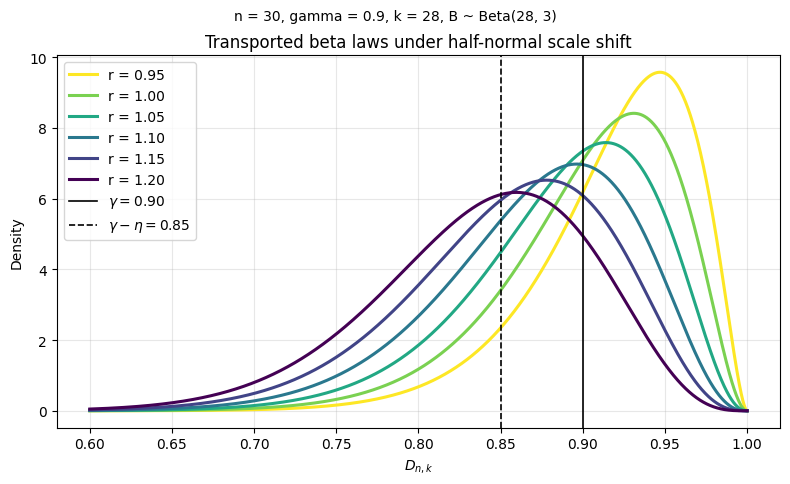}
    \caption{
    Transported beta laws under half-normal scale shift for
    \(n=30\), \(\gamma=0.9\), and \(k=28\), so that
    \(B_{n,k}\sim\operatorname{Beta}(28,3)\). Under the scale ratio
    \(r=\sigma_{\widetilde T}/\sigma_T\), the realized coverage satisfies
    \(D_{n,k}\stackrel d=h_r(B_{n,k})\). As \(r\) increases above one, the law
    is transported to the left, increasing the probability of bad calibration.
    The vertical lines mark \(\gamma=0.90\) and
    \(\gamma-\eta=0.85\), with \(\eta=0.05\).
    }
    \label{fig:half-normal-beta-deformation}
\end{figure}

For small scale shifts, writing \(r=e^\delta\), the local approximation derived in Appendix~\ref{app:half_normal_scale_shift_details} gives
\[
    \operatorname{Cov}(k_\gamma)-\gamma
    =
    -2\delta\,\phi(q(\gamma))q(\gamma)
    +
    O(\delta^2)+O(n^{-1}),
    \qquad
    q(\gamma)=\Phi^{-1}\!\left(\frac{\gamma+1}{2}\right).
\]
At \(\gamma=0.9\), the coefficient
\(2\phi(q(\gamma))q(\gamma)\) is approximately \(0.339\). Thus a \(10\%\) increase in the test-score scale, \(r=1.1\), corresponds to an approximate coverage loss of about \(0.032\), or roughly three percentage points.

Thus, in the half-normal scale-shift model, the abstract Wasserstein
coverage-gap bound becomes an exact identity: the transport distance from the
beta reference is precisely the marginal coverage loss, and the transported
law determines bad-calibration probabilities through its left tail
$\Pr{h_r(B_{n,k})\leq t}$.
Figure~\ref{fig:fig_halfnormal} in Appendix~\ref{app:half_normal_scale_shift_details}
makes this concrete across $r \in \{0.5, 0.8, 2.0\}$ and 
$n \in \{50, 200, 1000\}$: the shaded area between the CDFs grows with 
$|r-1|$ and is insensitive to $n$, while the right panel reveals that 
$W_1$ is asymmetric around $r=1$, with undercoverage ($r>1$) incurring 
a larger transport cost than overcoverage ($r<1$) for the same departure, 
reflecting the nonlinearity of $h_r$ and the concentration of $\beta_{n,k}$ 
near $\gamma$.

\subsection{Clustered Calibration and Effective Sample Size}

Proposition~\ref{prop:dependent_calibration_counting} has a simple concrete
instance in clustered or replicated calibration data. Think of \(b\) independent
experimental units, subjects, videos, or spatial locations, each producing \(m\)
nearly identical calibration scores, so that \(n=mb\). The following idealized
model takes the within-cluster dependence to be perfect. Let
\[
    V_1,\ldots,V_b
    \overset{\mathrm{iid}}{\sim}
    \operatorname{Unif}(0,1),
\]
and set
\[
    U_{j,r}
    =
    V_j,
    \qquad
    j=1,\ldots,b,\quad r=1,\ldots,m,
\]
and then relabel \(\{U_{j,r}\}\) as \(U_1,\ldots,U_n\). Each \(U_i\) is still
marginally uniform, but the effective number of
independent calibration units is \(b\), not \(n\). Indeed,
\[
    N_n(t)
    =
    \sum_{i=1}^n \mathbf{1}\{U_i\leq t\}
    =
    m\sum_{j=1}^b \mathbf{1}\{V_j\leq t\},
\]
so \(N_n(t)\stackrel d=mZ_{b,t}\), with
\(Z_{b,t}\sim\operatorname{Bin}(b,t)\).
Proposition~\ref{prop:dependent_calibration_counting} gives
\[
    W_1\!\left(
        \mathcal L(U_{(k)}),
        \beta_{n,k}
    \right)
    =
    \int_0^1
    \left|
        \Pr{Z_{b,t}\geq \lceil k/m\rceil}
        -
        \Pr{Z_{n,t}\geq k}
    \right|
    \,dt,
    \qquad
    Z_{n,t}\sim\operatorname{Bin}(n,t).
\]
Equivalently,
\[
    U_{(k)}
    =
    V_{(\lceil k/m\rceil)}
    \quad\text{a.s.},
\]
and therefore
\[
    U_{(k)}
    \sim
    \operatorname{Beta}
    \left(
        \lceil k/m\rceil,\,
        b+1-\lceil k/m\rceil
    \right).
\]
This example is an extreme cluster model and creates ties within the calibration
block. The counting identity and the Wasserstein comparison do not require
absence of ties, but the example can also be viewed as the limit of a model with
small within-cluster jitter. Its purpose is to isolate the practical effect:
using \(n\) highly clustered calibration points behaves, for the realized
coverage law, more like using \(b\) independent calibration units.

Thus, in the no-shift decoupled-test version of this application,
\(D_{n,k}\stackrel d=U_{(k)}\). If
\[
    \rho_{\mathrm{cl}}(k)
    :=
    W_1\!\left(
        \operatorname{Beta}
        \left(
            \lceil k/m\rceil,\,
            b+1-\lceil k/m\rceil
        \right),
        \beta_{n,k}
    \right),
\]
then Theorem~\ref{thm:w1_coverage_gap} gives
\[
    \left|
        \operatorname{Cov}(k)-\frac{k}{n+1}
    \right|
    \leq
    \rho_{\mathrm{cl}}(k).
\]
For \(k_\gamma=\lceil(n+1)\gamma\rceil\), the target level is therefore
degraded by at most \(\rho_{\mathrm{cl}}(k_\gamma)+1/(n+1)\).
Theorem~\ref{thm:bad_calibration} also gives, for every \(t\in[0,1]\) and
\(\varepsilon>0\),
\[
    \Pr{D_{n,k}\leq t}
    \leq
    \Pr{B_{n,k}\leq t+\varepsilon}
    +
    \frac{\rho_{\mathrm{cl}}(k)}{\varepsilon},
    \qquad
    B_{n,k}\sim\beta_{n,k}.
\]
In this example, both conclusions say that clustering affects conformal
coverage only through the transport from the nominal beta law based on \(n\)
points to the effective beta law based on \(b\) independent units.

\subsection{Stationary Mixing Processes}

We now apply Proposition~\ref{prop:quantile_be_beta_transport} to a stationary
score process \((T_i)_{i\in\mathbb Z}\). Let \(F_T\) be its continuous marginal
distribution function. The calibration sample is
\(\mathcal C_n=(T_1,\ldots,T_n)\), and the test score is taken \(\ell\) steps
after the calibration block:
\[
    \widetilde T=T_{n+\ell},
    \qquad \ell\geq 1.
\]
The transformed variables
$
    U_i=F_T(T_i)
$
are marginally uniform on \([0,1]\), but they are not independent in general.

Let \(\mathcal F_a^b=\sigma(T_i:a\leq i\leq b)\). We use the following
standard mixing coefficients:
\[
    \alpha_{\mathrm{mix}}(r)
    :=
    \sup_m
    \sup_{\substack{
        A\in \mathcal F_{-\infty}^m,\,
        B\in \mathcal F_{m+r}^{\infty}
    }}
    \left|
        \Pr{B\cap A}
        -
        \Pr{A}\Pr{B}
    \right|
\]
for strong, or \(\alpha\)-, mixing,
\[
    \phi_{\mathrm{mix}}(r)
    :=
    \sup_m
    \sup_{\substack{
        A\in \mathcal F_{-\infty}^m,\,
        B\in \mathcal F_{m+r}^{\infty}\\
        \Pr{A}>0
    }}
    \left|
        \Pr{B\mid A}
        -
        \Pr{B}
    \right|
\]
for \(\phi\)-mixing, and
\[
    \beta_{\mathrm{mix}}(r)
    :=
    \sup_m
    \E{
        \sup_{B\in\mathcal F_{m+r}^{\infty}}
        \left|
            \Pr{B\mid\mathcal F_{-\infty}^m}
            -
            \Pr{B}
        \right|
    }
\]
for absolute regularity, or \(\beta\)-mixing. Different normalizations of the
\(\beta\)-mixing coefficient appear in the literature; only its role as a
decoupling error is used below.

For fixed \(t\), the event \(\{T_{n+\ell}\leq t\}\) is separated from
\(\mathcal C_n\) by \(\ell\) time steps. Under \(\phi\)-mixing this gives the
direct conditional decoupling bound\footnote{Formally, we first
apply the mixing bound on the countable set \(t\in\mathbb Q\), choosing an
almost-sure event on which all these inequalities hold. For a regular
conditional distribution of \(T_{n+\ell}\) given \(\mathcal C_n\), the
conditional c.d.f.\ is right-continuous in \(t\); since \(F_T\) is continuous,
the bound extends from rational thresholds to all \(t\in\mathbb R\). This
almost-sure uniform version can then be evaluated at the random threshold
\(T_{(k)}\).}
\[
    \sup_{t\in\mathbb R}
    \left|
        \Pr{T_{n+\ell}\leq t\mid \mathcal C_n}
        -
        F_T(t)
    \right|
    \leq
    \phi_{\mathrm{mix}}(\ell)
    \qquad
    \text{a.s.}
\]
Evaluating this bound at the random threshold \(T_{(k)}\), and setting
\[
    D^{(\ell)}_{n,k}
    :=
    \Pr{T_{n+\ell}\leq T_{(k)}\mid \mathcal C_n},
\]
we obtain
\[
    \left|
        D^{(\ell)}_{n,k}
        -
        F_T(T_{(k)})
    \right|
    \leq
    \phi_{\mathrm{mix}}(\ell)
    \qquad
    \text{a.s.}
\]
Since \(F_T(T_{(k)})=U_{(k)}\), Proposition~\ref{prop:w1_coupling_upper_bound}
and the triangle inequality give
\[
    W_1\!\left(
        \mathcal L(D^{(\ell)}_{n,k}),
        \beta_{n,k}
    \right)
    \leq
    \phi_{\mathrm{mix}}(\ell)
    +
    W_1\!\left(
        \mathcal L(U_{(k)}),
        \beta_{n,k}
    \right).
\]
The first term is the price of using a future dependent test point; the second
is the deviation of the calibration order statistic from the i.i.d.\ beta
reference.

A similar separation can be obtained for absolutely regular processes by
blocking. Coupling lemmas for separated blocks, such as
\citet[Corollary~2.7]{yu1994rates}, allow one to replace separated dependent
blocks by independent copies at a cost controlled by the corresponding
\(\beta\)-mixing coefficients. We do not need the explicit block construction
in what follows. The relevant point is that, after the test-calibration
dependence is handled by a direct \(\phi\)-mixing argument, by a
\(\beta\)-mixing blocking argument, or by an external independence assumption,
the remaining problem is to control the calibration order statistic
\(U_{(k)}\).

We summarize the decoupling step by assuming that, for some deterministic
\(\Delta_\ell\geq0\),
\[
    W_1\!\left(
        \mathcal L(D^{(\ell)}_{n,k}),
        \mathcal L(U_{(k)})
    \right)
    \leq
    \Delta_\ell.
\]
For the direct \(\phi\)-mixing argument above one may take
\(\Delta_\ell=\phi_{\mathrm{mix}}(\ell)\); if the future test score has already
been decoupled from the calibration block, then \(\Delta_\ell=0\). It remains
to control \(\mathcal L(U_{(k)})\). The external input we use is a
Berry--Esseen theorem for sample quantiles of strongly mixing, equivalently
\(\alpha\)-mixing, sequences.

\begin{theorem}[Berry--Esseen Bound for \(\alpha\)-Mixing Sample Quantiles]
\label{thm:lahiri_sun_quantile_be}
Let \((X_i)_{i\in\mathbb Z}\) be a stationary strongly mixing
(\(\alpha\)-mixing) sequence with marginal distribution function \(F\). Fix
\(p\in(0,1)\), write \(F^{-1}\) for the generalized inverse of \(F\), and
suppose that \(F\) has density \(f\) near \(F^{-1}(p)\). Let
\[
    F_n(x):=\frac1n\sum_{i=1}^n\mathbf{1}\{X_i\leq x\}.
\]
We write \(F_n^{-1}\) for the generalized inverse of \(F_n\).
Assume the regularity conditions (C.1)--(C.5) of
\citet[Theorem, Section~2.3; see also Eq.~(1.4)]{lahiri_sun_2009}, including
the strong-mixing rate condition on \(\alpha_{\mathrm{mix}}\). Then
\[
    \sup_{x\in\mathbb R}
    \left|
        \Pr{
            \sqrt n\{F_n^{-1}(p)-F^{-1}(p)\}\leq x
        }
        -
        \Phi\!\left(\frac{x}{\tau_\infty(p)}\right)
    \right|
    =
    O(n^{-1/2}),
\]
where
\[
    \tau_\infty^2(p)
    =
    \frac{
        \sigma_\infty^2(F^{-1}(p))
    }{
        f^2(F^{-1}(p))
    },
\]
and
\[
    \sigma_\infty^2(F^{-1}(p))
    =
    \sum_{j\in\mathbb Z}
    \operatorname{Cov}
    \left(
        \mathbf{1}\{X_0\leq F^{-1}(p)\},
        \mathbf{1}\{X_j\leq F^{-1}(p)\}
    \right).
\]
\end{theorem}

Suppose that the transformed calibration process \((U_i)_{i\in\mathbb Z}\)
satisfies the assumptions of Theorem~\ref{thm:lahiri_sun_quantile_be} at
quantile level \(\gamma\). Its marginal distribution is uniform, so
\(F^{-1}(\gamma)=\gamma\), \(f(\gamma)=1\), and, with the generalized inverse
convention,
\(F_n^{-1}(\gamma)=U_{(\lceil n\gamma\rceil)}\). The conformal index
\(k_\gamma=\lceil(n+1)\gamma\rceil\) differs from \(\lceil n\gamma\rceil\) by
at most one, so the same Berry--Esseen approximation applies to
\(U_{(k_\gamma)}\), with the index perturbation absorbed in the
\(O(n^{-1/2})\) remainder. Thus the Berry--Esseen assumption in
Proposition~\ref{prop:quantile_be_beta_transport} holds for \(k_\gamma\) with
\(a_n=O(n^{-1/2})\) and
\[
    \tau_\gamma^2
    =
    \gamma(1-\gamma)
    +
    2\sum_{j\geq 1}
    \operatorname{Cov}
    \left(
        \mathbf{1}\{U_0\leq \gamma\},
        \mathbf{1}\{U_j\leq \gamma\}
    \right),
\]
whenever the covariance series is well defined and \(\tau_\gamma^2>0\).
Therefore,
\begin{align}
\label{eq:phi_mixing_w_bound}
    W_1\!\left(
        \mathcal L(D^{(\ell)}_{n,k_\gamma}),
        \beta_{n,k_\gamma}
    \right)
    &\leq
    \Delta_\ell
    +
    \sqrt{\frac{2}{\pi n}}
    \left|
        \tau_\gamma-\sqrt{\gamma(1-\gamma)}
    \right|
    +
    O(n^{-1/2}).
\end{align}

Thus Theorem~\ref{thm:w1_coverage_gap}, together with
\(|k_\gamma/(n+1)-\gamma|\leq 1/(n+1)\), gives the marginal coverage bound
\[
\begin{aligned}
    \left|
        \Pr{T_{n+\ell}\leq T_{(k_\gamma)}}
        -
        \gamma
    \right|
    &\leq
    \Delta_\ell
    +
    \sqrt{\frac{2}{\pi n}}
    \left|
        \tau_\gamma-\sqrt{\gamma(1-\gamma)}
    \right|
    +
    O(n^{-1/2}).
\end{aligned}
\]
Theorem~\ref{thm:bad_calibration} gives the corresponding bad-calibration
statement from the same radius. Let
\[
    \rho_{n,\ell}
    :=
    W_1\!\left(
        \mathcal L(D^{(\ell)}_{n,k_\gamma}),
        \beta_{n,k_\gamma}
    \right).
\]
The Wasserstein bound above shows that
\[
    \rho_{n,\ell}
    \leq
    \Delta_\ell
    +
    \sqrt{\frac{2}{\pi n}}
    \left|
        \tau_\gamma-\sqrt{\gamma(1-\gamma)}
    \right|
    +
    O(n^{-1/2}).
\]
Applied with \(p=1\), Theorem~\ref{thm:bad_calibration} gives that for every
\(t\in[0,1]\) and every
\(\varepsilon>0\),
\[
    \Pr{D^{(\ell)}_{n,k_\gamma}\leq t}
    \leq
    \Pr{B_{n,k_\gamma}\leq t+\varepsilon}
    +
    \frac{\rho_{n,\ell}}{\varepsilon},
    \qquad
    B_{n,k_\gamma}\sim\beta_{n,k_\gamma}.
\]
In particular, for any \(\eta\in(0,\gamma)\),
\[
    \Pr{D^{(\ell)}_{n,k_\gamma}\leq \gamma-\eta}
    \leq
    \Pr{B_{n,k_\gamma}\leq \gamma-\eta/2}
    +
    \frac{2\rho_{n,\ell}}{\eta}.
\]
Thus the probability of a low realized-coverage event is controlled by the
i.i.d.\ beta lower tail, plus the decoupling and Berry--Esseen errors encoded
in \(\rho_{n,\ell}\).

\subsubsection{Example: AR(1) process}

Let $(T_i)_{i \in \mathbb{Z}}$ be a stationary AR(1) process,
\begin{equation*}
    T_i = a\, T_{i-1} + \varepsilon_i, \qquad 
    \varepsilon_i \overset{\text{iid}}{\sim} N(0,\, 1-a^2), \quad |a| < 1. 
\end{equation*}
Then $T_i\sim N(0,1)$ for every $i$, but the scores are serially dependent. The calibration sample is \(\mathcal C_n=(T_1,\ldots,T_n)\), and the test score is \(T_{n+\ell}\) for a fixed horizon \(\ell\geq1\). By the Markov property,
\[
    T_{n+\ell}\mid\mathcal C_n
    \sim
    N(a^\ell T_n,\,1-a^{2\ell}).
\]
Therefore the realized coverage at threshold \(T_{(k)}\) is
\begin{equation}
\label{eq:ar1_Dnk}
    D^{(\ell)}_{n,k} 
    =
    \Pr{T_{n+\ell} \leq T_{(k)} \mid \mathcal C_n}
    = \Phi\!\left(\frac{T_{(k)} - a^{\ell} T_n}{\sqrt{1 - a^{2\ell}}}\right).
\end{equation}

For this example, the decoupling radius $\Delta_\ell$ in \eqref{eq:phi_mixing_w_bound} can be obtained directly from the Markov property without invoking mixing conditions: since $T_{n+\ell}\mid\mathcal C_n\sim N(a^\ell T_n,1-a^{2\ell})$ differs from
$N(0,1)$ only through a mean shift of order $|a|^\ell$, one obtains $\Delta_\ell \lesssim |a|^\ell$. The geometric decay used below thus arises
from the Markov structure of the AR(1) alone.

\begin{figure}[h]
    \centering
    \includegraphics[width=\linewidth]{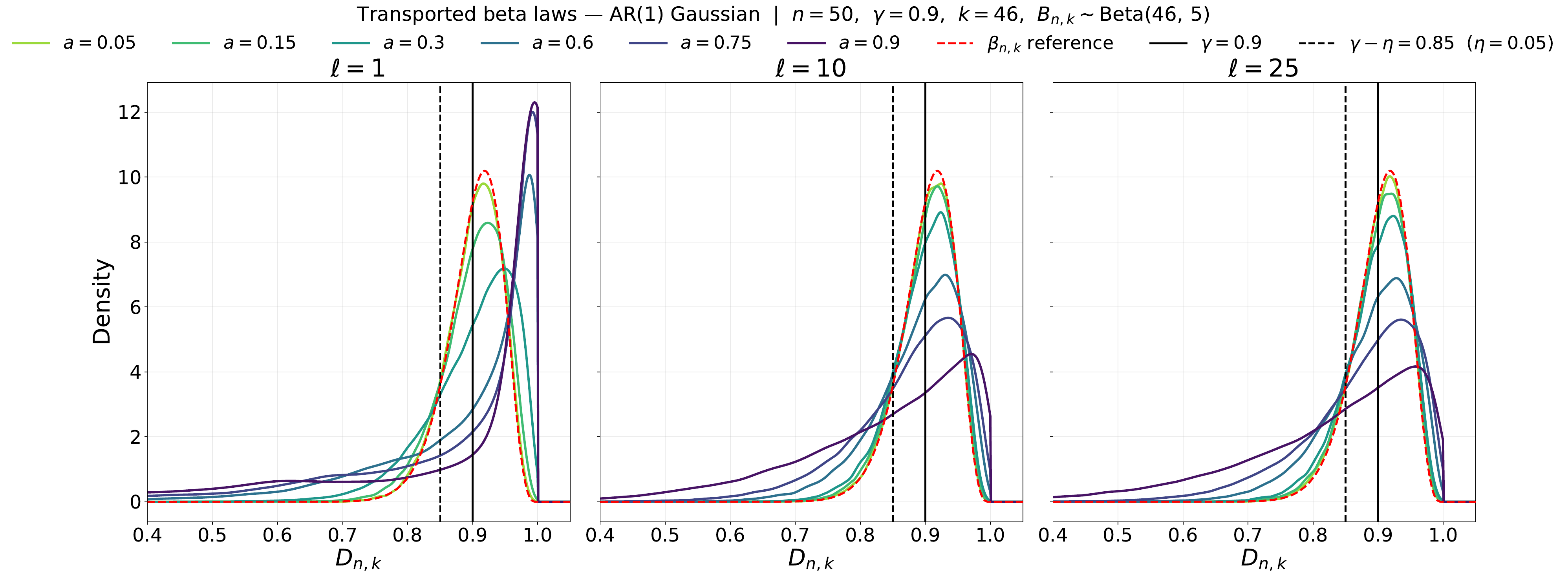}
    \caption{Realized-coverage laws \(\nu_{n,k}\) for the Gaussian AR(1) model with
\(n=50\) and \(\gamma=0.9\), against the i.i.d. beta reference \(\beta_{n,k}\) (dashed red).
At \(\ell=1\), larger positive values of \(a\) produce a more dispersed realized-coverage law, with both a sharper concentration near 1 and a heavier left tail extending below the bad-calibration threshold
\(\gamma-\eta=0.85\). As \(\ell\) grows, the direct test-calibration dependence weakens and the laws approach the calibration-order-statistic law \(\mathcal L(U_{(k)})\); this limit equals \(\beta_{n,k}\) only when the calibration scores are independent.}
    \label{fig:fig_ar1_cdfs}
\end{figure}

The expression contains two effects. The order statistic \(T_{(k)}\) is the
calibration contribution, while \(a^\ell T_n\) is the remaining dependence
between the future test score and the last calibration score. As \(\ell\)
increases, \(a^\ell\) decreases geometrically and the test score becomes closer
to an independent draw from the marginal \(N(0,1)\). In the limit
\(\ell\to\infty\),
\[
    D^{(\ell)}_{n,k}
    \to
    \Phi(T_{(k)})
    =
    U_{(k)}
\]
in distribution. The remaining deviation from the beta reference is then the
calibration-order-statistic effect controlled by the \(\alpha\)-mixing
Berry--Esseen argument above.

Figure~\ref{fig:fig_ar1_cdfs} shows the induced realized-coverage laws for
fixed \(n=50\) and \(\gamma=0.9\), across several values of \(a\) and
\(\ell\in\{1,10,25\}\). Larger \(a\) and smaller \(\ell\) produce stronger
test-calibration effects; increasing \(\ell\) removes this contribution and
moves the law toward the calibration-order-statistic law \(U_{(k)}\), which may
still differ from the beta benchmark when calibration dependence is strong.

The complementary roles of $\ell$ and $n$ for the AR(1) example are illustrated in Figures~\ref{fig:fig_ar1_bound_h},
\ref{fig:fig_ar1_thm7} and~\ref{fig:fig_ar1_coverage_bound} in
Appendix~\ref{app:ar_1_illustrations}. Figure~\ref{fig:fig_ar1_bound_h} displays the Monte Carlo estimate of $W_1(\nu_{n,k_\gamma},\beta_{n,k_\gamma})$ and the
bound~\eqref{eq:phi_mixing_w_bound} as functions of the horizon $\ell$, for fixed $n=200$ and $a\in\{0,\,0.3,\,0.6,\,0.9\}$. For $a=0$ the scores are
i.i.d.\ and $W_1$ is negligible for all $\ell$. For $a>0$, the test--calibration term $|a|^\ell$ decays geometrically, and both quantities flatten as $\ell$ grows: the bound flattens at the Berry--Esseen floor
$\sqrt{2/(\pi n)}\,|\tau_\gamma-\sqrt{\gamma(1-\gamma)}|$, while $W_1$ flattens at the (typically smaller) limiting value $W_1(\mathcal L(U_{(k_\gamma)}),\beta_{n,k_\gamma})$ associated with the calibration order statistic alone.

Figure~\ref{fig:fig_ar1_thm7} fixes $\ell\in\{1,\,10,\,25\}$ and varies $n$. The Monte Carlo curves confirm the marginal-coverage gap inequality of
Theorem~\ref{thm:w1_coverage_gap}, $|\operatorname{Cov}(k_\gamma)-k_\gamma/(n+1)|\leq W_1$. The asymptotic upper bound on $W_1$, on the other hand, is informative only once $n$ is large enough relative to $a$ and $\ell$ for the Berry--Esseen approximation to take effect: it is conservative for small $n$ combined with large $a$ and small $\ell$ (top-right panels), and becomes tighter once $|a|^\ell$ is small,
so that the bound and $W_1$ are jointly dominated by the Berry--Esseen calibration term.

Finally, Figure~\ref{fig:fig_ar1_coverage_bound} translates these Wasserstein distances into bad-calibration probabilities through Theorem~\ref{thm:bad_calibration}: the Markov-type tail control $\mathbb{P}(B_{n,k_\gamma}\le \gamma-\eta/2)+2W_1/\eta$ is loose, and may exceed one, whenever $W_1$ is comparable to or larger than $\eta$; it tightens progressively as $\ell$ increases and the test-dependence contribution to $W_1$ decays.

\section{Conclusion}

We introduced a transported-beta perspective on split conformal prediction. Instead of
viewing conformal validity only through marginal coverage, we studied the law of the
calibration-conditional coverage induced by the realized calibration sample. In the
continuous i.i.d.\ case this law is exactly \(\mathrm{Beta}(k,n+1-k)\), so the classical
split conformal coverage level is recovered as its mean, while its lower tail quantifies
bad-calibration events.

The main message is that this beta law remains a useful finite-sample reference beyond
the i.i.d. setting. By comparing the realized-coverage law \(\nu_{n,k}\) with the beta
benchmark \(\beta_{n,k}\) on the coverage scale, Wasserstein distances give direct
control of marginal coverage gaps and bad-calibration probabilities. This formulation
separates different mechanisms of non-i.i.d. behavior: test-side distribution shift
transports the beta law through a deterministic coverage map, whereas calibration
dependence changes the order-statistic law itself.

The examples illustrate how this viewpoint can be used as a diagnostic and comparison
tool rather than as a single new conformal algorithm. In scale-shift models the deformation
is explicit and the Wasserstein bound can be sharp; in clustered and mixing settings the
distance to the beta reference captures effective sample size and dependence effects. Simulations on dependent uniform processes confirm that the Berry--Esseen approximation
tracks the empirical Wasserstein distance closely, even at moderate sample sizes.
Future work could focus on estimating the transported-beta radius from data, extending
the framework to adaptive or weighted conformal procedures, and developing sharper finite-sample tail comparisons.

\newpage
\bibliography{references}

\newpage

\appendix
\section{Proofs}
\subsection{Proof of Proposition~\ref{prop:conditional_coverage_beta}}
\label{app:proof_conditional_coverage_beta}

\begin{proof}
    Conditional on the calibration sample $S_1,\ldots,S_n$, the order statistic
    $S_{(k)}$ is fixed. Since $S_{n+1}$ is independent of the calibration sample
    and has distribution $F_S$, the conditional coverage is simply the
    distribution function evaluated at the realized threshold:
    \[
        C_{n,k}
        =
        \Pr{
            S_{n+1}\leq S_{(k)}
            \,\middle|\,
            S_1,\ldots,S_n
        }
        =
        F_S(S_{(k)}).
    \]

    The remaining step is to identify the distribution of this random value as
    the calibration sample varies. Define
    \[
        U_i = F_S(S_i),
        \qquad i=1,\ldots,n.
    \]
    By the probability integral transform, the variables
    $U_1,\ldots,U_n$ are i.i.d.\ $\operatorname{Unif}(0,1)$. Since $F_S$ is
    monotone, the ordering of the $S_i$'s is preserved, and hence
    \[
        F_S(S_{(k)}) = U_{(k)}.
    \]
    The $k$-th order statistic of $n$ i.i.d.\ uniform random variables has
    distribution
    \[
        U_{(k)}
        \sim
        \operatorname{Beta}(k,n+1-k).
    \]
    Therefore,
    \[
        C_{n,k}
        =
        F_S(S_{(k)})
        \sim
        \operatorname{Beta}(k,n+1-k).
    \]
\end{proof}

\subsection{Proof of Theorem~\ref{thm:w1_coverage_gap}}
\label{app:proof_w1_coverage_gap}

\begin{proof}
The i.i.d.\ reference law satisfies
\[
    \Ex{Z\sim\beta_{n,k}}{Z}
    =
    \frac{k}{n+1}.
\]
On the other hand, by the tower property,
\[
    \operatorname{Cov}(k)
    =
    \E{D_{n,k}}
    =
    \Ex{Z\sim\nu_{n,k}}{Z}.
\]
Therefore the marginal coverage gap is
\[
    \left|
        \operatorname{Cov}(k)
        -
        \frac{k}{n+1}
    \right|
    =
    \left|
        \Ex{Z\sim\nu_{n,k}}{Z}
        -
        \Ex{Z\sim\beta_{n,k}}{Z}
    \right|.
\]
By Proposition~\ref{prop:w1_dual_mean_bound},
\[
    \left|
        \Ex{Z\sim\nu_{n,k}}{Z}
        -
        \Ex{Z\sim\beta_{n,k}}{Z}
    \right|
    \leq
    W_1(\nu_{n,k},\beta_{n,k}),
\]
which proves the first claim.

If \(\nu_{n,k}\in\mathcal B_p(\beta_{n,k},\rho)\), then
Proposition~\ref{prop:wasserstein_monotonicity} gives
\[
    W_1(\nu_{n,k},\beta_{n,k})
    \leq
    W_p(\nu_{n,k},\beta_{n,k})
    \leq
    \rho.
\]
This proves the beta-neighborhood bound.

If $k_\gamma=\lceil(n+1)\gamma\rceil\leq n$, then
\[
    \frac{k_\gamma}{n+1}\geq \gamma
\]
and
\[
    0
    \leq
    \frac{k_\gamma}{n+1}-\gamma
    \leq
    \frac{1}{n+1}.
\]
The lower bound and the final absolute-error bound follow immediately from the
beta-neighborhood bound and the discretization bound above.
\end{proof}

\subsection{Proof of Theorem~\ref{thm:bad_calibration}}
\label{app:proof_bad_calibration}

\begin{proof}
Let \(B_{n,k}\sim\beta_{n,k}\), and let
\[
    \Tmon
    =
    F_{\nu_{n,k}}^{-1}\circ F_{\beta_{n,k}}
\]
be the monotone transport map from \(\beta_{n,k}\) to \(\nu_{n,k}\). Define
\[
    D^\star_{n,k}:=\Tmon(B_{n,k}).
\]
By Proposition~\ref{prop:monotone_transport}, \(D^\star_{n,k}\sim\nu_{n,k}\).
Therefore \(D^\star_{n,k}\stackrel d=D_{n,k}\), and
\[
    \Pr{D_{n,k}\leq t}
    =
    \Pr{\Tmon(B_{n,k})\leq t}.
\]
Define
\[
    G_\varepsilon
    :=
    \left\{
        |\Tmon(B_{n,k})-B_{n,k}|\leq \varepsilon
    \right\}.
\]
On \(G_\varepsilon\), if \(\Tmon(B_{n,k})\leq t\), then
\[
    B_{n,k}
    \leq
    \Tmon(B_{n,k})+\varepsilon
    \leq
    t+\varepsilon.
\]
Thus
\[
    \{\Tmon(B_{n,k})\leq t\}\cap G_\varepsilon
    \subseteq
    \{B_{n,k}\leq t+\varepsilon\}.
\]
It follows that
\[
    \Pr{D_{n,k}\leq t}
    \leq
    \Pr{B_{n,k}\leq t+\varepsilon}
    +
    \Pr{G_\varepsilon^c}.
\]
By Markov's inequality and the optimality of \(\Tmon\),
\[
\begin{aligned}
    \Pr{G_\varepsilon^c}
    &=
    \Pr{
        |\Tmon(B_{n,k})-B_{n,k}|>\varepsilon
    } \\
    &\leq
    \frac{
        \E{\left|\Tmon(B_{n,k})-B_{n,k}\right|^p}
    }{
        \varepsilon^p
    } \\
    &=
    \frac{
        W_p^p(\nu_{n,k},\beta_{n,k})
    }{
        \varepsilon^p
    }
    \leq
    \left(\frac{\rho}{\varepsilon}\right)^p.
\end{aligned}
\]
Combining the last two displays proves the claim.
\end{proof}

\subsection{Proof of Proposition~\ref{prop:decoupled_test_scores}}
\label{app:proof_decoupled_test_scores}

\begin{proof}
Since \(\widetilde T\) is independent of \(\mathcal C_n\), conditioning on the
calibration sample fixes \(T_{(k)}\) and gives
\[
    D_{n,k}
    =
    \Pr{\widetilde T\leq T_{(k)}\mid\mathcal C_n}
    =
    F_{\widetilde T}(T_{(k)}).
\]
This proves the first display.

Now suppose that the calibration scores have common continuous marginal
distribution function \(F_T\), and define \(U_i=F_T(T_i)\). By the probability
integral transform, each \(U_i\) is marginally uniform on \([0,1]\); no
independence among the \(U_i\)'s is required for this marginal statement. Since
\(F_T\) is monotone, the ordering is preserved, so
\[
    F_T(T_{(k)})=U_{(k)}.
\]
With the generalized inverse convention, this gives
\[
    F_{\widetilde T}(T_{(k)})
    \stackrel d=
    F_{\widetilde T}(F_T^{-1}(U_{(k)}))
    =
    h(U_{(k)}),
    \qquad
    h:=F_{\widetilde T}\circ F_T^{-1}.
\]
Therefore
\[
    \nu_{n,k}
    =
    \mathcal L(D_{n,k})
    =
    h_{\#}\mathcal L(U_{(k)}).
\]
In the i.i.d.\ uniform reference case, the order statistic satisfies
\(U_{(k)}\sim\beta_{n,k}\), and the final display follows.
\end{proof}

\subsection{Proof of Proposition~\ref{prop:dependent_calibration_counting}}
\label{app:proof_dependent_calibration_counting}

\begin{proof}
The event \(U_{(k)}\leq t\) is exactly the event that at least \(k\) calibration
variables are at most \(t\), which gives the first identity. For the i.i.d.\
uniform reference sample, the corresponding count is
\(\operatorname{Bin}(n,t)\), and its \(k\)-th order statistic has law
\(\beta_{n,k}\). The Wasserstein identity then follows from
Proposition~\ref{prop:w1_one_dimensional_formula}.
\end{proof}

\subsection{Details for the Half-Normal Scale-Shift Example}
\label{app:half_normal_scale_shift_details}

For \(\sigma>0\), the distribution function of \(|N(0,\sigma^2)|\) is
\[
    F_\sigma(t)
    =
    2\Phi(t/\sigma)-1,
    \qquad t\geq 0.
\]
Hence
\[
    F_\sigma^{-1}(u)
    =
    \sigma\,\Phi^{-1}\!\left(\frac{u+1}{2}\right),
    \qquad 0<u<1.
\]
Writing
\[
    q(u):=\Phi^{-1}\!\left(\frac{u+1}{2}\right),
\]
we have \(F_\sigma^{-1}(u)=\sigma q(u)\). Therefore the calibration-to-test
transport map is
\[
\begin{aligned}
    F_{\sigma_{\widetilde T}}
    \left(
        F_{\sigma_T}^{-1}(u)
    \right)
    &=
    F_{\sigma_{\widetilde T}}
    \left(
        \sigma_T q(u)
    \right) \\
    &=
    2\Phi\!\left(
        \frac{\sigma_T}{\sigma_{\widetilde T}}q(u)
    \right)-1  \\
    &=
    2\Phi\!\left(
        \frac{q(u)}{r}
    \right)-1
    =
    h_r(u),
\end{aligned}
\]
where \(r=\sigma_{\widetilde T}/\sigma_T\).

Let
\[
    B_{n,k}:=F_{\sigma_T}(T_{(k)}).
\]
Since the calibration sample is i.i.d.\ from \(F_{\sigma_T}\),
\[
    B_{n,k}\sim \operatorname{Beta}(k,n+1-k)=\beta_{n,k}.
\]
Moreover,
\[
\begin{aligned}
    D_{n,k}
    &=
    \Pr{
        \widetilde T\leq T_{(k)}
        \,\middle|\,
        T_1,\ldots,T_n
    } \\
    &=
    F_{\sigma_{\widetilde T}}(T_{(k)}) \\
    &=
    F_{\sigma_{\widetilde T}}
    \left(
        F_{\sigma_T}^{-1}(B_{n,k})
    \right) \\
    &=
    h_r(B_{n,k}).
\end{aligned}
\]
Thus
\[
    \nu_{n,k}^{(r)}
    =
    \mathcal L(D_{n,k})
    =
    (h_r)_\#\beta_{n,k},
\]
and
\[
    \operatorname{Cov}(k)
    =
    \mathbb E[h_r(B_{n,k})].
\]

We next derive the density of the transported law. Let \(z\in(0,1)\). If
\(z=h_r(u)\), then
\[
    z
    =
    2\Phi\!\left(\frac{q(u)}{r}\right)-1.
\]
Therefore,
\[
    \Phi^{-1}\!\left(\frac{z+1}{2}\right)
    =
    \frac{q(u)}{r},
\]
and hence
\[
    q(u)=rq(z).
\]
It follows that
\[
    h_r^{-1}(z)
    =
    2\Phi(rq(z))-1.
\]
Differentiating with respect to \(z\),
\[
\begin{aligned}
    \frac{d}{dz}h_r^{-1}(z)
    &=
    2\phi(rq(z))\,r\,q'(z).
\end{aligned}
\]
Since
\[
    q(z)=\Phi^{-1}\!\left(\frac{z+1}{2}\right),
\]
we have
\[
    q'(z)
    =
    \frac{1}{2\phi(q(z))}.
\]
Therefore,
\[
    \frac{d}{dz}h_r^{-1}(z)
    =
    r\frac{\phi(rq(z))}{\phi(q(z))}.
\]
If \(f_{n,k}\) denotes the density of
\(\beta_{n,k}=\operatorname{Beta}(k,n+1-k)\), then the density of
\(D_{n,k}=h_r(B_{n,k})\) is
\[
\begin{aligned}
    f_{D,r}(z)
    &=
    f_{n,k}\!\left(h_r^{-1}(z)\right)
    \left|
        \frac{d}{dz}h_r^{-1}(z)
    \right|  \\
    &=
    f_{n,k}\!\left(
        2\Phi(rq(z))-1
    \right)
    r\frac{\phi(rq(z))}{\phi(q(z))},
    \qquad 0<z<1.
\end{aligned}
\]

We now justify the sharpness of the Wasserstein bound in this example. The map
\(h_r\) is increasing on \((0,1)\). Moreover, since \(q(u)>0\) for
\(u\in(0,1)\),
\[
    r>1 \implies h_r(u)<u,
    \qquad
    r<1 \implies h_r(u)>u.
\]
Thus \(h_r-\operatorname{id}\) has constant sign on \((0,1)\). The monotone
coupling \((B_{n,k},h_r(B_{n,k}))\) is optimal in one dimension, and therefore
\[
\begin{aligned}
    W_1(\nu_{n,k}^{(r)},\beta_{n,k})
    &=
    \mathbb E\!\left[
        |h_r(B_{n,k})-B_{n,k}|
    \right]  \\
    &=
    \left|
        \mathbb E[h_r(B_{n,k})-B_{n,k}]
    \right|  \\
    &=
    \left|
        \operatorname{Cov}(k)-\frac{k}{n+1}
    \right|.
\end{aligned}
\]
The last equality uses \(\mathbb E[B_{n,k}]=k/(n+1)\).

Finally, we derive the local small-shift approximation. Write \(r=e^\delta\).
Then
\[
    h_{e^\delta}(u)
    =
    2\Phi(e^{-\delta}q(u))-1.
\]
Differentiating with respect to \(\delta\),
\[
    \frac{\partial}{\partial\delta}h_{e^\delta}(u)
    =
    -2e^{-\delta}\phi(e^{-\delta}q(u))q(u).
\]
At \(\delta=0\),
\[
    \left.
    \frac{\partial}{\partial\delta}h_{e^\delta}(u)
    \right|_{\delta=0}
    =
    -2\phi(q(u))q(u).
\]
Thus
\[
    h_{e^\delta}(u)
    =
    u
    -
    2\delta\,\phi(q(u))q(u)
    +
    O(\delta^2).
\]
Taking expectations under \(B_{n,k}\sim\beta_{n,k}\) gives
\[
    \operatorname{Cov}(k)
    =
    \frac{k}{n+1}
    -
    2\delta\,
    \mathbb E\!\left[
        \phi(q(B_{n,k}))q(B_{n,k})
    \right]
    +
    O(\delta^2).
\]
For \(k_\gamma=\lceil(n+1)\gamma\rceil\), the beta law concentrates around
\(\gamma\). Therefore,
\[
    \mathbb E\!\left[
        \phi(q(B_{n,k_\gamma}))q(B_{n,k_\gamma})
    \right]
    =
    \phi(q(\gamma))q(\gamma)
    +
    O(n^{-1}),
\]
and hence
\[
    \operatorname{Cov}(k_\gamma)-\gamma
    =
    -2\delta\,\phi(q(\gamma))q(\gamma)
    +
    O(\delta^2)
    +
    O(n^{-1}).
\]

\subsection{Proof of Proposition~\ref{prop:quantile_be_beta_transport}}
\label{app:proof_quantile_be_beta_transport}

\begin{proof}
Write
\[
    s_\gamma:=\sqrt{\gamma(1-\gamma)}.
\]
Let \(B_{n,k_\gamma}\sim\beta_{n,k_\gamma}\), and denote by \(F_U\) and \(F_B\) the
distribution functions of \(U_{(k_\gamma)}\) and \(B_{n,k_\gamma}\), respectively. Since
both laws are supported on \([0,1]\), Proposition~\ref{prop:w1_one_dimensional_formula}
gives
\[
    W_1\!\left(
        \mathcal L(U_{(k_\gamma)}),
        \beta_{n,k_\gamma}
    \right)
    =
    \int_0^1
    \left|
        F_U(t)-F_B(t)
    \right|
    dt.
\]

We now rewrite the two distribution functions on the central-limit scale. For
\(x\in\mathbb R\), set
\[
\begin{aligned}
    A_n(x)
    &:=
    \Pr{
        \sqrt n\{U_{(k_\gamma)}-\gamma\}\leq x
    },
    \\
    A_n^0(x)
    &:=
    \Pr{
        \sqrt n\{B_{n,k_\gamma}-\gamma\}\leq x
    }.
\end{aligned}
\]
By assumption,
\[
    \sup_{x\in\mathbb R}
    \left|
        A_n(x)
        -
        \Phi\!\left(\frac{x}{\tau_\gamma}\right)
    \right|
    \leq
    a_n.
\]
For the i.i.d.\ beta benchmark, the classical Berry--Esseen theorem for sample
quantiles gives
\[
    \sup_{x\in\mathbb R}
    \left|
        A_n^0(x)
        -
        \Phi\!\left(\frac{x}{s_\gamma}\right)
    \right|
    \leq
    c_n,
    \qquad
    c_n=O(n^{-1/2}).
\]
This is the standard Berry--Esseen approximation for the \(k_\gamma\)-th order
statistic of an i.i.d.\ uniform sample. Since
\(k_\gamma=\lceil(n+1)\gamma\rceil\), we have
\(k_\gamma/(n+1)=\gamma+O(n^{-1})\); this centering difference is absorbed by
the \(O(n^{-1/2})\) remainder.

Now fix \(t\in[0,1]\). Then
\[
    F_U(t)=A_n(\sqrt n(t-\gamma)),
    \qquad
    F_B(t)=A_n^0(\sqrt n(t-\gamma)).
\]
Using the triangle inequality inside the integral,
\[
\begin{aligned}
    \left|
        F_U(t)-F_B(t)
    \right|
    &\leq
    \left|
        A_n(\sqrt n(t-\gamma))
        -
        \Phi\!\left(\frac{\sqrt n(t-\gamma)}{\tau_\gamma}\right)
    \right| \\
    &\quad+
    \left|
        \Phi\!\left(\frac{\sqrt n(t-\gamma)}{\tau_\gamma}\right)
        -
        \Phi\!\left(\frac{\sqrt n(t-\gamma)}{s_\gamma}\right)
    \right| \\
    &\quad+
    \left|
        \Phi\!\left(\frac{\sqrt n(t-\gamma)}{s_\gamma}\right)
        -
        A_n^0(\sqrt n(t-\gamma))
    \right|.
\end{aligned}
\]
Integrating over \(t\in[0,1]\) and using the two uniform Berry--Esseen bounds
therefore yields
\[
\begin{aligned}
    W_1\!\left(
        \mathcal L(U_{(k_\gamma)}),
        \beta_{n,k_\gamma}
    \right)
    &\leq
    a_n
    +
    c_n
    +
    I_n,
\end{aligned}
\]
where
\[
    I_n
    :=
    \int_0^1
    \left|
        \Phi\!\left(\frac{\sqrt n(t-\gamma)}{\tau_\gamma}\right)
        -
        \Phi\!\left(\frac{\sqrt n(t-\gamma)}{s_\gamma}\right)
    \right|
    dt.
\]

It remains to identify \(I_n\). With the change of variables
\(x=\sqrt n(t-\gamma)\),
\[
\begin{aligned}
    I_n
    &=
    \frac1{\sqrt n}
    \int_{-\sqrt n\gamma}^{\sqrt n(1-\gamma)}
    \left|
        \Phi\!\left(\frac{x}{\tau_\gamma}\right)
        -
        \Phi\!\left(\frac{x}{s_\gamma}\right)
    \right|
    dx \\
    &\leq
    \frac1{\sqrt n}
    \int_{\mathbb R}
    \left|
        \Phi\!\left(\frac{x}{\tau_\gamma}\right)
        -
        \Phi\!\left(\frac{x}{s_\gamma}\right)
    \right|
    dx.
\end{aligned}
\]
The functions \(x\mapsto\Phi(x/\tau_\gamma)\) and \(x\mapsto\Phi(x/s_\gamma)\)
are the distribution functions of \(N(0,\tau_\gamma^2)\) and
\(N(0,s_\gamma^2)\). Thus the one-dimensional \(W_1\) formula on
\(\mathbb R\) gives
\[
    \int_{\mathbb R}
    \left|
        \Phi\!\left(\frac{x}{\tau_\gamma}\right)
        -
        \Phi\!\left(\frac{x}{s_\gamma}\right)
    \right|
    dx
    =
    W_1\!\left(
        N(0,\tau_\gamma^2),
        N(0,s_\gamma^2)
    \right).
\]
By Lemma~\ref{lem:w1_centered_normals},
\[
    W_1\!\left(
        N(0,\tau_\gamma^2),
        N(0,s_\gamma^2)
    \right)
    =
    \sqrt{\frac{2}{\pi}}
    |\tau_\gamma-s_\gamma|.
\]
Consequently,
\[
    I_n
    \leq
    \sqrt{\frac{2}{\pi n}}
    \left|
        \tau_\gamma-\sqrt{\gamma(1-\gamma)}
    \right|.
\]

Combining this with \(c_n=O(n^{-1/2})\) gives
\[
    W_1\!\left(
        \mathcal L(U_{(k_\gamma)}),
        \beta_{n,k_\gamma}
    \right)
    \leq
    a_n
    +
    \sqrt{\frac{2}{\pi n}}
    \left|
        \tau_\gamma-\sqrt{\gamma(1-\gamma)}
    \right|
    +
    O(n^{-1/2}),
\]
which is the first claim. If \(a_n=O(n^{-1/2})\), then the term \(a_n\) is
absorbed into the final remainder, giving the second display.
\end{proof}

\section{Example additional results}
This section presents additional illustrations for the two worked examples 
in the main text. In both cases, $W_1(\nu_{n,k}, \beta_{n,k})$, the 
coverage gap and the transported laws $\nu_{n,k}$ are estimated via Monte Carlo with $50{,}000$ simulations per configuration, by drawing directly from the underlying process in each example and computing the relevant quantities empirically. Shaded bands throughout indicate $\pm 2$ standard errors across simulations. In general, standard errors are very low; the wider band for the coverage gap observed in Figure \ref{fig:fig_ar1_thm7} near zero is a consequence of the logarithmic scale employed, as the absolute 
error remains stable while the gap itself approaches zero.

\subsection{Half-normal illustrations and results}
Figure~\ref{fig:fig_halfnormal} complements the density plot in the main 
text by showing the full CDF deformation across a wider range of $r$ and 
$n$. The left panels make the shaded $W_1$ area directly visible for each 
$(r, n)$ pair, confirming that the deformation is stable across calibration 
sizes. The right panel traces $W_1$ and the coverage gap jointly as 
functions of $r$, verifying the exact equality from both sides of $r = 1$ 
and across all three values of $n$.
\begin{figure}[h]
    \centering
    \includegraphics[width=\linewidth]{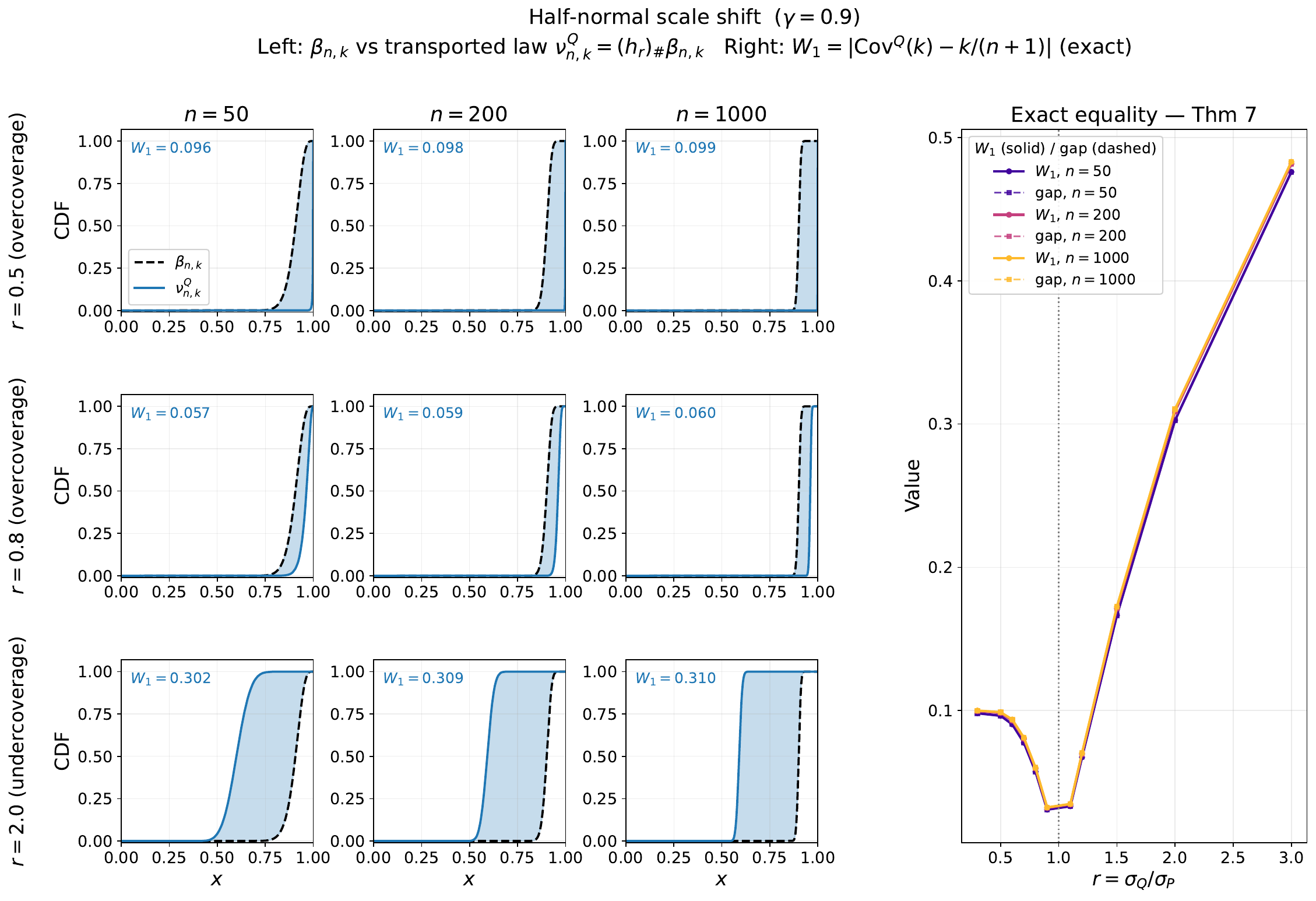}
    \caption{
    Half-normal scale-shift example with $\gamma = 0.9$. \textbf{Left:} CDFs of 
$\nu_{n,k}^{(r)} = (h_r)_\#\beta_{n,k}$ (solid blue) against $\beta_{n,k}$ 
(dashed black) for $r \in \{0.5, 0.8, 2.0\}$ and $n \in \{50, 200, 1000\}$; 
the shaded area equals $W_1(\nu_{n,k}^{(r)}, \beta_{n,k})$ and is 
essentially constant across $n$. \textbf{Right:} exact identity 
$W_1(\nu_{n,k}^{(r)}, \beta_{n,k}) = |\operatorname{Cov}(k) - k/(n+1)|$ 
verified across $r \in [0.3, 3.0]$; the asymmetry around $r = 1$ reflects 
the nonlinearity of $h_r$, with undercoverage ($r > 1$) incurring a larger 
transport cost than overcoverage ($r < 1$) for the same $|r - 1|$.}
    \label{fig:fig_halfnormal}
\end{figure}

\subsection{AR(1) illustrations and results}
\label{app:ar_1_illustrations}
Figures~\ref{fig:fig_ar1_bound_h}, \ref{fig:fig_ar1_thm7}, 
and~\ref{fig:fig_ar1_coverage_bound} complement Figure~\ref{fig:fig_ar1_cdfs} 
by quantifying the deformation across a wider range of configurations.
Figure~\ref{fig:fig_ar1_bound_h} traces $W_1$ and the 
bound~\eqref{eq:phi_mixing_w_bound} as functions of $\ell$ for fixed 
$n = 200$, showing the geometric decay and the residual Berry--Esseen floor.
Figure~\ref{fig:fig_ar1_thm7} varies $n$ for fixed $\ell \in \{1, 10, 25\}$, 
confirming the chain coverage gap $\leq W_1 \leq$ Berry--Esseen bound across 
all configurations.
Figure~\ref{fig:fig_ar1_coverage_bound} translates these into 
bad-calibration probabilities, with the bound tightening as $\ell$ increases.

\begin{figure}[h]
    \centering
    \includegraphics[width=\linewidth]{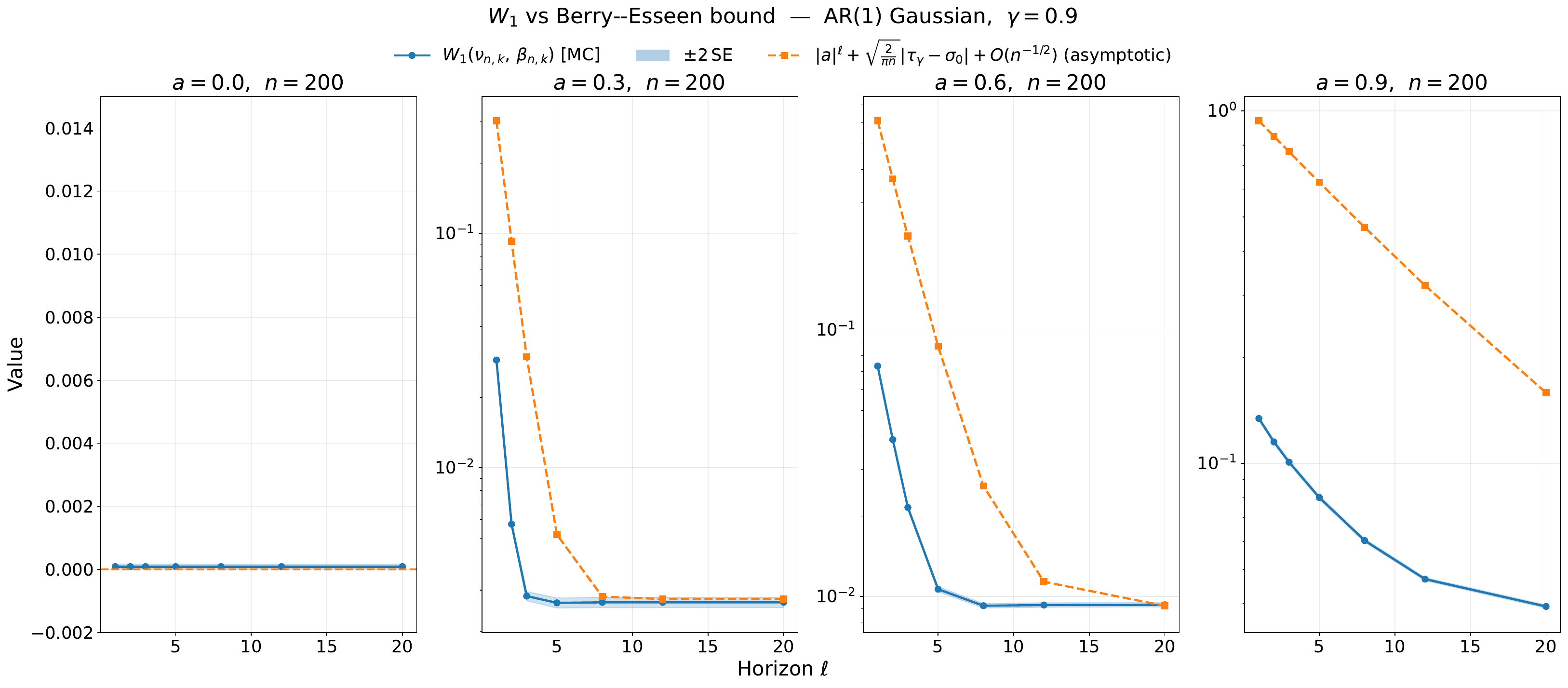}
    \caption{$W_1(\nu_{n,k_\gamma}, \beta_{n,k_\gamma})$ (solid blue) and the 
asymptotic bound~\eqref{eq:phi_mixing_w_bound} (dashed orange) as functions 
of $\ell$, for $n = 200$, $\gamma = 0.9$, and $a \in \{0, 0.3, 0.6, 0.9\}$. 
For $a = 0$, $W_1$ is negligible for all $\ell$; for $a > 0$, both decay 
geometrically and level off at the Berry--Esseen floor, which depends only 
on $n$ and the long-run variance $\tau_\gamma^2$.}
    \label{fig:fig_ar1_bound_h}
\end{figure}

\begin{figure}[h]
    \centering
    \includegraphics[width=\linewidth]{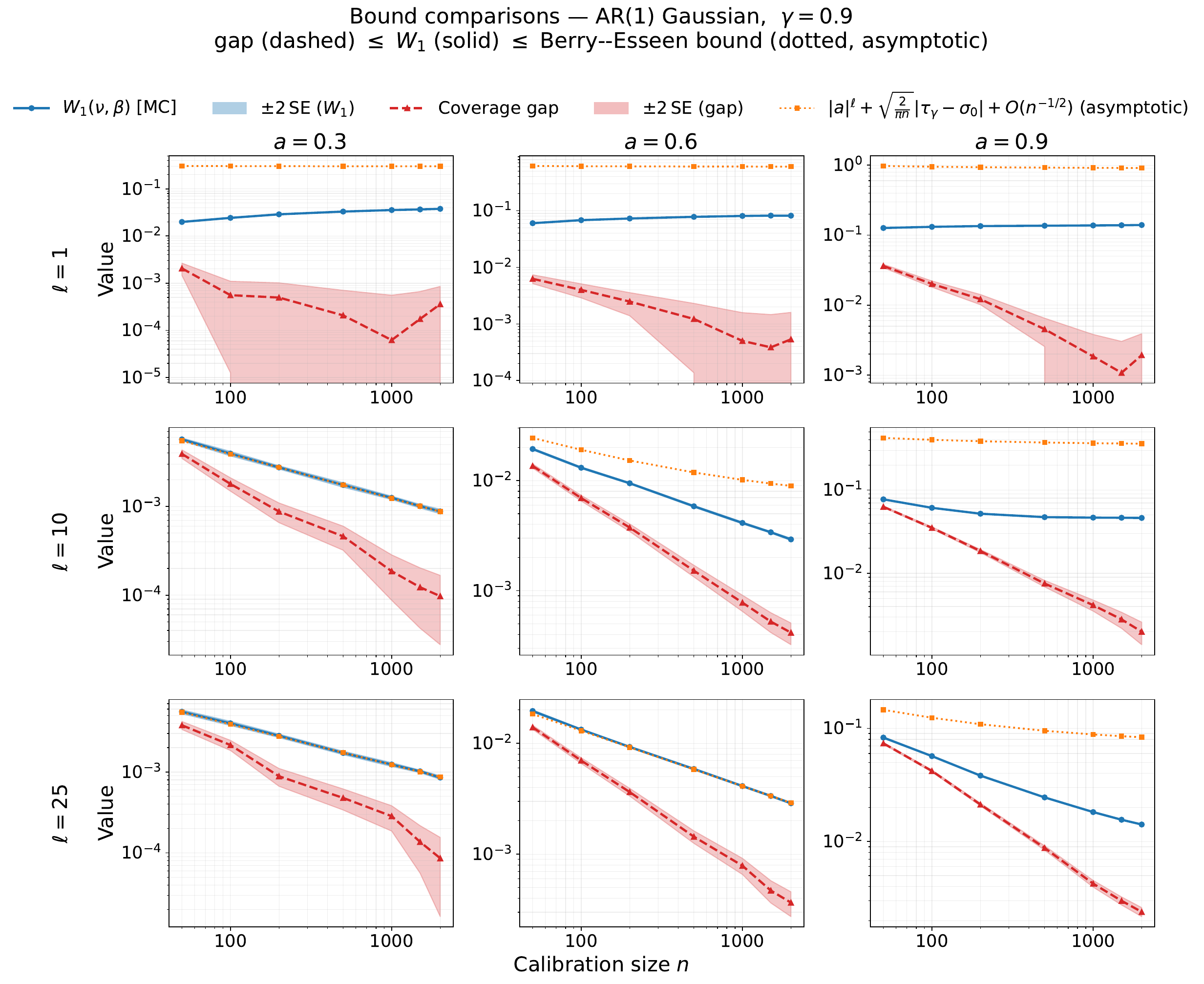}
    \caption{Chain of bounds $|\operatorname{Cov}(k_\gamma) - k_\gamma/(n+1)| 
\leq W_1(\nu_{n,k_\gamma}, \beta_{n,k_\gamma}) \leq |a|^\ell + 
\sqrt{2/(\pi n)}\,|\tau_\gamma - \sqrt{\gamma(1-\gamma)}| + O(n^{-1/2})$ 
as a function of $n$, for $\gamma = 0.9$, $a \in \{0.3, 0.6, 0.9\}$, and 
$\ell \in \{1, 10, 25\}$. Solid blue: Monte-Carlo $W_1$; dashed red: 
coverage gap; dotted orange: asymptotic bound. The chain holds throughout 
the asymptotic regime; for small $n$ with large $a$ and small $\ell$ the 
bound is conservative, and tightens as $|a|^\ell$ decays.}
    \label{fig:fig_ar1_thm7}
\end{figure}

\begin{figure}[h]
    \centering
    \includegraphics[width=\linewidth]{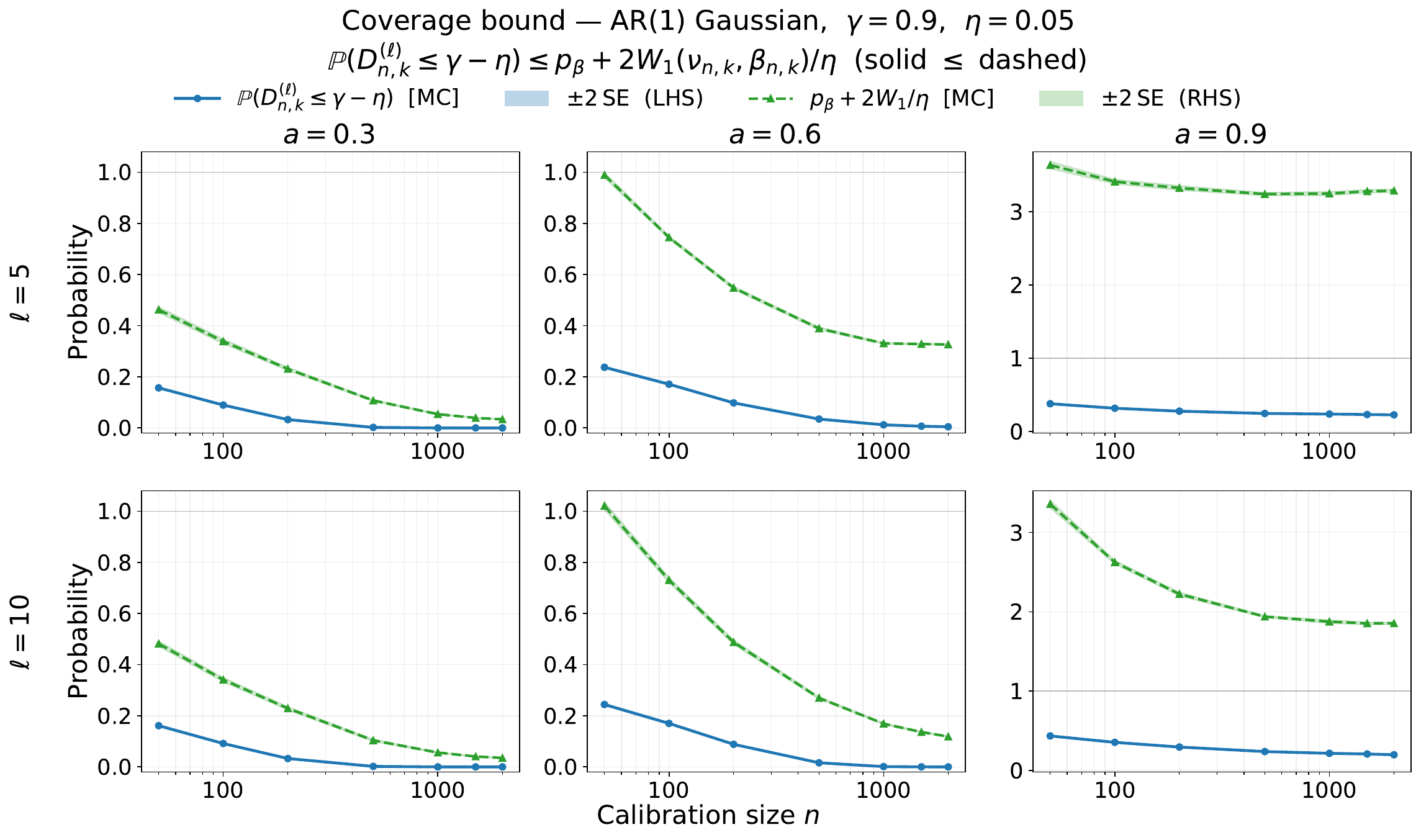}
    \caption{Bad-calibration bound from Theorem~\ref{thm:bad_calibration}: $\mathbb{P}(D^{(\ell)}_{n,k_\gamma} \leq \gamma - \eta)$ (solid blue)  against $\mathbb{P}(B_{n,k_\gamma} \leq \gamma - \eta/2) + 2W_1/\eta$ 
(dashed green) as a function of $n$, for $\gamma = 0.9$, $\eta = 0.05$, 
$a \in \{0.3, 0.6, 0.9\}$, and $\ell \in \{5, 10\}$. The bound can exceed 
one for large $a$ and small $\ell$, where the Markov penalty $2W_1/\eta$ 
dominates, and tightens as $\ell$ grows and $W_1$ decays.}
    \label{fig:fig_ar1_coverage_bound}
\end{figure}

\end{document}